\documentclass[letterpaper]{article} 

\usepackage[hyperref,abbrvnat]{jmlr_adapted}
\usepackage{graphicx}
\usepackage{natbib} 
\usepackage{appendix}
\usepackage{xcolor}
\usepackage{bm}
\usepackage{natbib}
\usepackage{amsmath,amssymb,amsthm,graphicx,url}
\usepackage{mathtools}
\usepackage[ruled,vlined]{algorithm2e} 
\usepackage{nicefrac}
\usepackage{booktabs}
\usepackage{tabularray}
\usepackage{microtype}
\usepackage{bbding}
\usepackage{eucal}
\usepackage[capitalize,noabbrev]{cleveref}

\renewcommand{\cite}{\citep}

\usepackage{tabularray}
\UseTblrLibrary{booktabs}

\newcommand{\E}{\mathbb{E}}

\newcommand{\states}{\mathcal{S}}
\newcommand{\actions}{\mathcal{A}}
\newcommand{\opt}{^\star}
\renewcommand{\exp}[1]{\operatorname{exp}\left( #1\right) }

\newcommand{\tr}{^{\top}}
\newcommand{\Real}{\mathbb{R}}

\renewcommand{\exp}[1]{\operatorname{exp}\left(#1\right)}

\DeclareMathOperator*{\argmax}{argmax}
\DeclareMathOperator{\ess}{ess}

\DeclareMathOperator{\ext}{ext}

\usepackage{tikz}
\usetikzlibrary{arrows.meta,automata}

\DeclareMathOperator{\cvaro}{CVaR}
\DeclareMathOperator{\evaro}{EVaR}
\DeclareMathOperator{\ermo}{ERM}

\newcommand{\PiHR}{\Pi_{\mathrm{HR}}}
\newcommand{\PiMR}{\Pi_{\mathrm{MR}}}
\newcommand{\PiMD}{\Pi_{\mathrm{MD}}}
\newcommand{\PiSR}{\Pi_{\mathrm{SR}}}
\newcommand{\PiSD}{\Pi_{\mathrm{SD}}}

\renewcommand{\exp}[1]{\operatorname{exp}\left( #1\right) }
\newcommand{\probs}[1]{\Delta_{#1}}  

\newcommand{\BellI}{T}
\newcommand{\BellIE}{L}

\renewcommand{\P}{\mathbb{P}}

\newcommand{\Ex}[1]{\mathbb{E}\left[ #1 \right]}

\newcommand{\erm}[2]{\ermo_{#1}\left[#2\right]}

\newcommand{\evar}[2]{\evaro_{#1} \left[#2\right]}
\newcommand{\ermp}[3]{\ermo_{#1}^{#2}\left[#3\right]}

\theoremstyle{plain}
\newtheorem{theorem}{Theorem}[section]
\newtheorem{proposition}[theorem]{Proposition}
\newtheorem{lemma}[theorem]{Lemma}
\newtheorem{corollary}[theorem]{Corollary}
\theoremstyle{definition}

\newtheorem{assumption}[theorem]{Assumption}
\theoremstyle{remark}

\title{Risk-averse Total-reward MDPs with ERM and EVaR}
\author{%
\name Xihong Su \email xihong.su@unh.edu\\
\addr
University of New Hampshire\\
Durham, NH, USA
\AND
\name Julien Grand-Cl\'ement
\email grand-clement@hec.fr \\
\addr HEC Paris \\
Paris, France
\AND
\name Marek Petrik \email mpetrik@cs.unh.edu \\
\addr
University of New Hampshire\\
Durham, NH, USA
}

\begin{document}

\maketitle

\begin{abstract}
Optimizing risk-averse objectives in discounted MDPs is challenging because most models do not admit direct dynamic programming equations and require complex history-dependent policies. In this paper, we show that the risk-averse {\em total reward criterion}, under the Entropic Risk Measure (ERM) and Entropic Value at Risk (EVaR) risk measures, can be optimized by a stationary policy, making it simple to analyze, interpret, and deploy. We propose exponential value iteration, policy iteration, and linear programming to compute optimal policies. Compared with prior work, our results only require the relatively mild condition of transient MDPs and allow for {\em both} positive and negative rewards. Our results indicate that the total reward criterion may be preferable to the discounted criterion in a broad range of risk-averse reinforcement learning domains. 
\end{abstract}

\section{Introduction}
\label{sec:intro} 

Risk-averse Markov decision processes~(MDP)~\cite{puterman205markov} that use \emph{monetary risk measures} as their objective have been gaining in popularity in recent years~\cite{Kastner2023, Marthe2023a, lam2022risk, li2022quantile, bauerle2022markov, hau2023entropic, Hau2023a,suevar,su2024optimality}. Risk-averse objectives, such as Value at Risk~(VaR), Conditional Value at Risk~(CVaR), Entropic Risk Measure~(ERM), or Entropic Value at Risk~(EVaR), penalize the variability of returns~\cite{Follmer2016stochastic}. As a result, these risk measures yield policies with stronger guarantees on the probability of catastrophic losses, which is important in domains like healthcare or finance. 

In this paper, we target the \emph{total reward criterion} (TRC)~\cite{Kallenberg2021markov,puterman205markov} instead of the common discounted criterion. TRC also assumes an infinite horizon but does not discount future rewards. To control for infinite returns, we assume that the MDP is {\em transient}, i.e. that there is a positive probability that the process terminates after a finite number of steps, an assumption commonly used in the TRC literature~\cite{filar2012competitive}. We consider the TRC with both positive and negative rewards. When the rewards are non-positive, the TRC is equivalent to the \emph{stochastic shortest path} problem, and when they are non-negative, it is equivalent to the \emph{stochastic longest path}~\cite{Dann2023}.

Two reasons motivate our departure from discounted objectives in risk-averse MDPs. First, considering risk affects discounted objectives significantly. It is common to use discounted objectives because they admit optimal stationary policies and value functions that can be computed using dynamic programs. However, most risk-averse discount objectives, such as VaR, CVaR, or EVaR, require that optimal policies are \emph{history-dependent}~\cite{bauerle2011markov,Hau2023a,hau2023entropic} and do not admit standard dynamic programming optimality equations.

Second, TRC captures the concept of stochastic termination, which is common in reinforcement learning~\cite{Sutton2018}. In risk-neutral objectives, discounting can serve well to model the probability of termination because it guarantees the same optimal policies~\cite{puterman205markov,su2023solving}. However, as we show in this work, no such correspondence exists with risk-averse objectives, and the difference between them may be arbitrarily significant. Modeling stochastic termination using a discount factor in {\em risk-averse} objectives is inappropriate and leads to dramatically different optimal policies. 

As our main contribution, we show that the risk-averse TRC with ERM and EVaR risk measures admit optimal stationary policies and optimal value functions in transient MDPs. We also show that the optimal value function satisfies dynamic programming equations and can be computed with exponential value iteration, policy iteration, or linear programming algorithms. These algorithms are simple and closely resemble the algorithms for solving MDPs.

Our results indicate that EVaR is a particularly interesting risk measure in reinforcement learning. ERM and the closely related exponential utility functions have been popular in sequential decision-making problems because they admit dynamic programming decompositions~\cite{patek1999stochastic, deFreitas2020,smith2023exponential, denardo1979optimal, hau2023entropic, Hau2023a}. Unfortunately, ERM is difficult to interpret; it is scale-dependent; and it is incomparable with popular risk measures like VaR and CVaR. Because EVaR reduces to an optimization over ERM, it preserves most of the computational advantages of ERM, and since EVaR closely approximates CVaR and VaR at the same risk level, its value is also much easier to interpret. Finally, EVaR is also a coherent risk measure, unlike ERM~\cite{Ahmadi-Javid2012, ahmadi2017analytical}. 

\begin{table}
\centering
\begin{tabular}{lcccc} 
\toprule
& \multicolumn{2}{c}{Risk properties} & \multicolumn{2}{c}{Optimal policy} \\
\cmidrule(rl){2-3}  \cmidrule(rl){4-5} 
Risk measure & Coherent & Law inv. & Disc. & TRC \\  
\midrule
$\E$   & yes & yes & S & S\\
\midrule
EVaR   & yes & yes & M & \textbf{S} \\
ERM    & no & yes & M & \textbf{S} \\
NCVaR  & yes & no & S & S \\
VaR    & yes & yes & H & H \\
CVaR   & yes & yes & H & H \\
\bottomrule
\end{tabular}
\caption{Structure of optimal policies in
  risk-averse MDPs: ``S'', ``M'' and ``H'' refer to Stationary, Markov and History-dependent policies respectively.}
\label{fig:policy-comparison}
\end{table}
   
\Cref{fig:policy-comparison} puts our contribution in the context of other work on risk-averse MDP objectives. Optimal policies for VaR and CVaR are known to be history-dependent in the discounted objective~\cite{bauerle2011markov,Hau2023a} and must be history-dependent in TRC because TRC generalizes the finite-horizon objective. The TRC with Nested risk measures, such as Nested CVaR~(NCVaR), applies the risk measure in each level of the dynamic program independently and preserves most of the favorable computational properties of risk-neutral MDPs~\cite{ahmadi2021risk}. Unfortunately, nested risk measures are difficult to interpret; their value depends on the sequence in which the rewards are obtained in a complex and unpredictable way~\cite{Kupper2006} and may be unbounded even if MDPs are transient. 

While we are unaware of prior work on the TRC objective with ERM or EVaR risk-aversion {\em allowing both positive and negative rewards}, the ERM risk measure is closely related to exponential utility functions. Prior work on TRC with exponential utility functions also imposes constraints on the sign of the instantaneous rewards, such as all positive rewards~\cite{blackwell1967positive} or all negative rewards~\cite{bertsekas1991analysis, freire2016extreme, carpin2016risk, de2020risk, fei2021exponential, fei2021risk,  ahmadi2021risk, cohen2021minimax, meggendorfer2022risk}. Disallowing a mix of positive and negative rewards limits the modeling power of prior work because it requires that either all states are more desirable or all states are less desirable than the terminal state. Allowing rewards with mixed signs raises some technical challenges, which we address by employing a squeeze argument that takes advantage of MDP's transience. 

\textbf{Notation.} We use a tilde to mark random variables, e.g. $\tilde{x}$. Bold lower-case letters represent vectors, and upper-case bold letters represent matrices. Sets are either calligraphic or upper-case Greek letters. The symbol $\mathbb{X}$ represents the space of real-valued random variables. When a function is defined over an index set, such as $z\colon \{1,2, \dots, N \} \to \Real$, we also treat it interchangeably as a vector $\bm{z}\in \Real^n$ such that $z_i = z(i), \forall i = 1, \dots , n$. Finally, $\Real, \Real_+, \Real_{++}$ denote real, non-negative real, and positive real numbers, respectively. $\bar{\Real} = \Real \cup \{-\infty, \infty\}$. Given a finite set $\mathcal{Y}$, the probability simplex is $\Delta_{\mathcal{Y}}:=\{x \in \Real_+^{\mathcal{Y}} \mid  \bm{1}^{\mathrm{T}}x=1\}$.

\section{Background on risk-averse MDPs}
\label{sec:preliminaries}

\paragraph{Markov Decision Processes}
We focus on solving Markov decision processes~(MDPs)~\cite{puterman205markov}, modeled by a tuple $(\bar{\states}, \actions, \bar{p}, \bar{r}, \bm{\bar{\mu}})$, where $\bar{\states} = \{1,2, \dots , S, S+1\}$ is the finite set of states and $\actions=\{1,2,\ldots, A\}$ is the finite set of actions. The transition function $\bar{p}\colon \bar{\states} \times \actions \to \probs{\bar{\states}}$ represents the probability $\bar{p}(s, a, s')$ of transitioning to $s'\in \bar{\mathcal{S}}$ after taking $a\in \mathcal{A}$ in $s\in \bar{\mathcal{S}}$ and $\bm{\bar{p}}_{sa} \in \probs{\bar{\states}}$ is such that $(\bm{\bar{p}}_{sa})_{s'} = \bar{p}(s,a,s')$. The function $\bar{r}\colon \bar{\states} \times  \actions \times  \bar{\states} \to \Real$ represents the reward $\bar{r}(s,a,s') \in \Real$ associated with transitioning from $s\in \bar{\states}$ and $a\in \mathcal{A}$ to $s'\in \bar{\states}$. The vector $\bm{\bar{\mu}} \in \probs{\bar{\states}}$ is the initial state distribution. 

We designate the state $e := S + 1$ as a \emph{sink state} and use $\mathcal{S} = \left\{ 1, \dots , S \right\}$ to denote the set of all non-sink states. The sink state $e$ must satisfy that $\bar{p}(e,a,e) = 1$ and $\bar{r}(e,a,e) = 0$ for each $a\in \mathcal{A}$, and $\bar{\mu}_e = 0$. Throughout the paper, we use a bar to indicate whether the quantity involves the sink state $e$. Note that the sink state can indicate a goal when all rewards are negative and an undesirable terminal state when all rewards are positive. 

The following technical assumption is needed to simplify the derivation. To lift the assumption, one needs to carefully account for infinite values, which adds complexity to the results and distracts from the main ideas.
\begin{assumption} \label{asm:positive-initial}
  The initial distribution $\bm{\mu}$ satisfies that
  \[
   \bm{\mu} > \bm{0}.  
  \]
\end{assumption}

The solution to an MDP is a {\em policy}. Given a horizon $t \in \mathbb{N}$, a history-dependent policy in the set $\PiHR^t$  maps the history of states and actions to a distribution over actions. A \emph{Markov policy} $\pi \in \PiMR^t$ is a sequence of decision rules $\pi = (\bm{d}_0, \bm{d}_1, \dots , \bm{d}_{t-1})$ with $\bm{d}_k\colon  \states \to  \probs{\actions}$ the decision rule for taking actions at time $k$. The set of all \emph{randomized decision rules} is $\mathcal{D} = (\probs{\actions})^\states$. \emph{Stationary policies} $\PiSR$ are Markov policies with $\pi := (\bm{d})_{\infty} := (\bm{d}, \bm{d}, \dots)$ with the identical decision rule in every timestep. We treat decision rules and stationary policies interchangeably. The sets of \emph{deterministic} Markov and stationary policies are denoted by $\PiMD^t$ and $\PiSD$. Finally, we omit the superscript $t$ to indicate infinite horizon definitions of policies.  

The risk-neutral Total Reward Criterion~(TRC) objective is:
\begin{equation} \label{eq:objective-expected}
  \sup_{\pi\in \PiHR} \liminf_{t\to \infty}
  \E^{\pi,\bm{\mu}} \left[ \sum_{k = 0}^{t-1} r(\tilde{s}_k, \tilde{a}_k, \tilde{s}_{k+1} ) \right],
\end{equation}
where the random variables are denoted by a tilde and $\tilde{s}_k$ and $\tilde{a}_k$ represent the state from $\bar{\states}$ and action at time $k$.  The superscript $\pi$ denotes the policy that governs the actions $\tilde{a}_k$ when visiting $\tilde{s}_{k}$ and $\bm{\mu}$ denotes the initial distribution.
Finally, note that $\liminf$ gives a conservative estimate of a policy's return since the limit does not necessarily exist for non-stationary policies.

Unlike the discounted criterion, the risk-neutral TRC may be unbounded, optimal policies may not exist, or may be non-stationary~\cite{bertsekas2013stochastic,james2006analysis}. To circumvent these issues, we assume that all policies have a positive probability of eventually transitioning to the sink state. 
\begin{assumption} \label{def:transient-mdp}
The MDP is \emph{transient} for any $\pi\in \Pi_{\mathrm{SD}}$:
  \begin{equation} \label{eq:transient-condition}
    \sum_{t = 0}^{\infty} \P^{\pi,s}\left[\tilde{s}_t = s'\right] \;<\;  \infty, \qquad
    \forall s,s'\in \mathcal{S} .  
\end{equation}
\end{assumption}

\Cref{def:transient-mdp} underlies most of our results. Transient MDPs are important because their optimal policies exist and can be chosen to be stationary deterministic~\citep[theorem~4.12]{Kallenberg2021markov}. Transient MDPs are also common in stochastic games~\cite{filar2012competitive} and generalize the stochastic shortest path problem~\cite{bertsekas2013stochastic}. 

An important tool in their analysis is the \emph{spectral radius} $\rho\colon \Real^{n \times  n} \to  \Real$ which is defined for each $\bm{A}\in \Real^{n \times  n}$ as the maximum absolute eigenvalue: $\rho(\bm{A}) := \max_{i=1, \dots , n} |\lambda_i|$ where $\lambda_i$ is the $i$-th eigenvalue~\cite{Horn2013}.
\begin{lemma}[Theorem 4.8~in~\citet{Kallenberg2021markov}] \label{lem:transient-spectral-radius}
An MDP is transient if and only if $\rho(\bm{P}^{\pi}) < 1$ for all $\pi\in \PiSR$.
\end{lemma}

Now, let us understand the differences between a discounted MDP and a transient MDP, which are useful in demonstrating the behavior of risk-averse objectives. Consider the MDPs in \cref{fig:discounted-transient-mdp}. There is one non-sink state $s$ and one action $a$. A triple tuple represents an action, transition probability, and a reward separately. Note that every discounted MDP can be converted to a transient MDP as described in \citet[appendix B]{su2024stationary}. For the discounted MDP, the discount factor is $\gamma$. For the transient MDP, $e$ is the sink state, and there is a positive probability $1-\epsilon$ of transiting from state $s$ to state $e$. Once the agent reaches the state $e$, it stays in $e$. For the risk-neutral objective, if $\gamma$ equals $\epsilon$, their value functions have identical values. 
However, for risk-aversion objectives, such as ERM, we show that the value functions in a discounted MDP can diverge from those in a transient MDP in \cref{sec:numerical-eval}. 

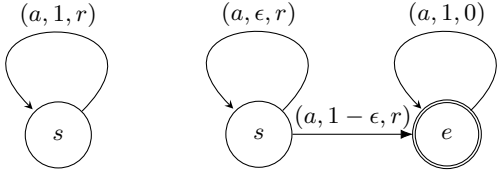
\begin{figure}
  \centering
\begin{tikzpicture}[->,-Latex,>=stealth,font=\small,node distance=25mm,el/.style = {inner sep=2pt, align=left, sloped}]
\node[ state] (1) {$s$};
\draw  
  (1) edge[loop,above] node[el,above]{$(a,1,r)$} (1);
\end{tikzpicture}
\begin{tikzpicture}[->,-Latex,>=stealth,font=\small,node distance=25mm,el/.style = {inner sep=2pt, align=left, sloped}]
\node[state] (1) {$s$};
\node[state, right of=1,double] (4) {$e$};
\draw  
  (1) edge[loop,above] node[el,above]{$(a,\epsilon, r)$} (1)
  (1) edge[ left, below] node[el,above,pos=0.5]{$(a,1-\epsilon,r)$} (4)
  (4) edge[loop,above] node[el,above]{$(a,1,0)$} (4);
\end{tikzpicture}
 \caption{Left: a discounted MDP, Right: a transient MDP}
 \label{fig:discounted-transient-mdp}
\end{figure}

\paragraph{Monetary risk measures}
Monetary risk measures aim to generalize the expectation operator to account for the spread of the random variable. 
\emph{Entropic risk measure}~(ERM) is a popular risk measure, defined for any risk level $\beta > 0$ and $\tilde{x}\in \mathbb{X}$ as~\cite{Follmer2016stochastic}
\begin{align} \label{eq:defn_ent_risk}
  \erm{\beta}{\tilde{x}} \;=\; - \beta^{-1} \cdot \log  \E \exp{-\beta\cdot  \tilde{x}} .
\end{align}
and extended to $\beta \in [0, \infty]$ as $\ermo_0[\tilde{x}] = \lim_{\beta \to 0^{+}} \erm{\beta}{\tilde{x}} = \E[\tilde{x}] $ and $\ermo_{\infty}[\tilde{x}] = \lim_{\beta\to \infty} \erm{\beta}{\tilde{x}} = \operatorname{ess} \inf[\tilde{x}]$.
ERM plays a unique role in sequential decision-making because it is the only law-invariant risk measure that satisfies the tower property (e.g., \citet[proposition A.1]{su2024stationary}), which is essential in constructing dynamic programs~\cite{hau2023entropic}. Unfortunately, two significant limitations of ERM hinder its practical applications. First, ERM is not positively homogenous and, therefore, the risk value depends on the scale of the rewards, and ERM is not coherent~\cite{Follmer2016stochastic, hau2023entropic, Ahmadi-Javid2012}. Second, the risk parameter $\beta$ is challenging to interpret and does not relate well to other standard risk measures, like VaR or CVaR.

For these reasons, we focus on the \emph{Entropic Value at Risk}~(EVaR), defined as, for a given $\alpha \in (0,1)$,
\begin{equation}\label{eq:evar-def-app}
  \begin{aligned}
    \evar{\alpha}{\tilde{x}}
   & \;=\;
    \sup_{\beta>0}  -\beta^{-1} \log \left(\alpha^{-1} \E \exp{ -\beta \tilde{x} } \right) \\
   & \;=\;
    \sup_{\beta>0}  \erm{\beta}{\tilde{x}} + \beta^{-1} \log  \alpha,
  \end{aligned}
\end{equation}
and extended to $\evar{0}{\tilde{x}} = \ess \inf [\tilde{x}]$ and $\evar{1}{\tilde{x}} = \Ex{\tilde{x}}$~\cite{Ahmadi-Javid2012}. It is important to note that the supremum in~\eqref{eq:evar-def-app} may not be attained even when $\tilde{x}$ is a finite discrete random variable~\cite{ahmadi2017analytical}.

EVaR addresses the limitations of ERM while preserving its benefits. EVaR is coherent and positively homogenous. EVaR is also a good approximation to interpretable quantile-based risk measures, like VaR and CVaR~\cite{Ahmadi-Javid2012,hau2023entropic}. 

\paragraph{Risk-averse MDPs.}
Risk-averse MDPs, using static VaR and CVaR risk measures, under the discounted criterion received abundant attention~\cite{Hau2023a, bauerle2011markov, bauerle2022markov, pflug2016time, li2022quantile}, showing that these objectives require history-dependent optimal policies. In contrast, nested risk measures under the TRC may admit stationary policies that can be computed using dynamic programming~\cite{ahmadi2021risk,meggendorfer2022risk,de2020risk,gavriel2012risk}. However, the TRC with nested CVaR can be unbounded~\citep[proposition C.1]{su2024stationary}. Recent work has shown that optimal Markov policies exist for EVaR discounted objectives, and they can be computed via dynamic programming~\cite{hau2023entropic}, building upon similar results established for ERM~\cite{Chung1987}. However, in TRC with ERM, the value functions may also be unbounded~\citep[proposition D.1]{su2024stationary}.

\section{Solving ERM Total Reward Criterion}
\label{sec:ssp-with-erm}

This section shows that an optimal stationary policy exists for ERM-TRC and that the value function satisfies dynamic programming equations. We then outline algorithms for computing it.

Our objective in this section is to maximize the ERM-TRC objective for some given $\beta > 0$ defined as
\begin{equation} \label{eq:sup-erm}
  \sup_{\pi\in \PiHR}  \liminf_{t\to \infty}
  \ermp{\beta}{\pi,\bm{\mu}}
   {\sum_{k=0}^{t-1} r(\tilde{s}_k,\tilde{a}_k,\tilde{s}_{k+1})}.
 \end{equation}
The definition employs limit inferior because the limit may not exist for non-stationary policies. 
Return functions $g_t\colon \PiHR \times \Real_{++} \to  \Real$ and $g_t\opt \colon \Real_{++} \to \Real$ for a horizon $t \in \mathbb{N}$ and the infinite-horizon versions $g_t\colon \PiHR \times \Real_{++} \to  \bar{\Real}$ and $g_t\opt \colon \Real_{++} \to \bar{\Real}$ are defined 
\begin{equation}
  \label{eq:g-definitions}
  \begin{aligned}
  g_t(\pi, \beta)  &:= \ermp{\beta}{\pi,\bm{\mu}}
  {\sum_{k=0}^{t-1} r(\tilde{s}_k,\tilde{a}_k,\tilde{s}_{k+1})}, \\
    g_t\opt(\beta) &:= \sup_{\pi\in \PiHR} g_t(\pi, \beta), \\
  g_{\infty}(\pi, \beta) &:= \liminf_{t\to \infty} g_t(\pi, \beta), \\
   g_{\infty}\opt(\beta) &:= \liminf_{t\to \infty} g_t\opt (\beta). 
  \end{aligned}
\end{equation}
Note that the functions $g_{\infty}$ and $g_{\infty}\opt$ can return infinite values and that~\eqref{eq:sup-erm} differs from  $g\opt_{\infty}$ in the order of the limit and the supremum. Finally,  when $\beta = 0$, we assume that all $g$ functions are defined as the expectation.
In the remainder of the section, we assume that the risk level $\beta > 0$ is fixed and omit it in notations when its value is unambiguous from the context. 

\subsection{Finite Horizon }

We commence the analysis with definitions and basic properties for the finite horizon criterion. To the best of our knowledge, this analysis is original in the context of the ERM but builds on similar approaches employed in the study of exponential utility functions.

Finite-horizon functions $\bm{v}^t( \pi ) \in \Real^{S}$ and $\bm{v}^{t,\star}\in \Real^{S}$ are defined for each horizon $t \in \mathbb{N}$ and policy $\pi\in \PiMD$, $s\in \mathcal{S}$ as 
\begin{equation}
\label{eq:erm-value-function}
\begin{aligned}
      v^t_s(\pi)
 & \;:=\; \ermp{\beta}{\pi,s}{\sum_{k=0}^{t-1} r(\tilde{s}_k,\tilde{a}_k,\tilde{s}_{k+1})},\\
  v^{t,\star}_s & \;:=\; \max_{\pi\in \PiMD} v^t_s(\pi),
\end{aligned}
\end{equation}
and $v_{e}^t(\pi) := 0$.

Because the nonlinearity of ERM complicates the analysis, it will be convenient to instead rely on \emph{exponential value function} $\bm{w}^{t}(\pi) \in \mathbb{R}^{S}$ for $\pi\in \PiMD$, $t \in \mathbb{N}$, and $s\in \states$ that satisfy
\begin{align} \label{eq:exp-value-function}
w_s^t(\pi)
&:= - \exp {  -\beta \cdot  v_s^t(\pi)  }, \\
v^t_s(\pi)
&= - \beta^{-1} \log (-  w^t_s(\pi)).
\end{align}
The optimal $\bm{w}^{t,\star} \in \mathbb{R}^{S}$ is defined analogously from $\bm{v}^{t,\star}$. Note that $\bm{w}^t < \bm{0}$ (componentwise) and $\bm{w}^0(\pi) = \bm{w}^{0,\star} = -\bm{1}$ for any $\pi \in \PiMD$. Similar exponential value functions have been used previously in exponential utility function objectives~\cite{denardo1979optimal,patek2001terminating}, in the analysis of robust MDPs, and even in regularized MDPs (see \citet{Grand-Clement2022} and references therein).

One can define a corresponding \emph{exponential Bellman operator} for any $\bm{w}\in \Real^S$ as
\begin{equation} \label{eq:exp-bellman-definition}
  \begin{aligned}
  \BellIE^{\bm{d}} \bm{w} &\;:=\;
  \bm{B}^{\bm{d}} \bm{w} - \bm{b}^{\bm{d}},
  \\
  \BellIE\opt  \bm{w} &\;:=\;
  \max_{\bm{d}\in \mathcal{D}} L^{\bm{d}} \bm{w} = \max_{\bm{d}\in \ext \mathcal{D}} L^{\bm{d}} \bm{w},
  \end{aligned}
\end{equation}
where $\ext \mathcal{D}$ is the set of extreme points of $\mathcal{D}$ corresponding to deterministic decision rules and $\bm{B}^{\bm{d}} \in \Real_+^{S \times   S}$ and $\bm{b}^{\bm{d}} \in  \Real_+^{S}$ are defined for $s, s'\in \mathcal{S}$ and $\bm{d}\in \mathcal{D}$ as
\begin{subequations} \label{eq:exponential-definitions}
\begin{align}
B_{s,s'}^{\bm{d}}
& :=
\sum_{a\in \mathcal{A}}  p(s, a, s') \cdot d_a(s) \cdot e^{-\beta \cdot  r(s,a,s')} ,\\
b_{s}^{\bm{d}} 
&:=
\sum_{a\in \mathcal{A}}  p(s, a, e) \cdot d_a(s) \cdot e^{-\beta \cdot  r(s,a,e)} .
\end{align}
\end{subequations}

The following theorem shows that $L$ can be used to compute $\bm{w}$. We use the shorthand notation $\pi_{1{:}t{-}1} = (\bm{d}_1, \dots , \bm{d}_{t-1}) \in \PiMR^{t-1}$ to denote the tail of $\pi$ that starts with $\bm{d}_1$ instead of $\bm{d}_0$.
\begin{theorem}\label{thm:exponential-bellman}
  For each $t =1, \dots $, and $\pi = (\bm{d}_0, \dots , \bm{d}_{t-1}) \in \PiMR^t$, the exponential values satisfy that
\[
\begin{aligned}
\bm{w}^t(\pi) &= L^{\bm{d}_t} \bm{w}^{t-1}(\pi_{1{:}t{-}1}),
 & \bm{w}^0(\pi) &= - \bm{1},
\\
\bm{w}^{t, \star } &= L\opt \bm{w}^{t-1,\star } = \bm{w}^t(\pi\opt) \ge \bm{w}^t(\pi),
& \bm{w}^{0,\star } &= -\bm{1},
\end{aligned}
\]
for some $\pi\opt \in \PiMD^t$.
\end{theorem}
The proof of \cref{thm:exponential-bellman} is standard and has been established both in the context of ERMs~\cite{hau2023entropic} and utility functions~\cite{patek1997stochastic}.

The following corollary follows directly from \cref{thm:exponential-bellman} by algebraic manipulation and by the monotonicity of exponential value function transformation and the ERM.
\begin{corollary} 
\label{coro:g-pi-beta}
We have that 
\begin{align*} 
g_t(\pi, \beta)  &= \ermo^{\bm{\mu}}_{\beta}\left[ v^t_{\tilde{s}_0}(\pi) \right] ,\\
  g_t\opt(\beta) &= \ermo^{\bm{\mu}}_{\beta}\left[ v^{t,\star}_{\tilde{s}_0} \right]
                   = \max_{\pi\in \PiMD} \ermo^{\bm{\mu}}_{\beta}\left[ v^t_{\tilde{s}_0} (\pi) \right].
\end{align*}
\end{corollary}

\subsection{Infinite Horizon}

We now turn to construct infinite-horizon optimal policies as a limiting case of the finite horizon. An important quantity is the infinite-horizon exponential value function defined for each $\pi\in \PiHR$ as
\[
  \bm{w}^{\infty}(\pi)
  \; :=\;  \liminf_{t \to \infty} \bm{w}^t(\pi),
  \quad
  \bm{w}^{\infty,\star}
  \; :=\;  \liminf_{t \to \infty} \bm{w}^{t,\star}.
\]
Note again that we use the inferior limit because the limit may not be defined for non-stationary policies. The limiting infinite-horizon value functions $\bm{w}^{\infty}(\pi)$ and $\bm{w}^{\infty,\star}$ are defined analogously from $\bm{v}^t(\pi)$ and $\bm{v}^{t,\star}$ using the inferior limit. 
The following theorem is the main result of this section. It shows that for an infinite horizon, the optimal exponential value function is attained by a stationary deterministic policy and is a fixed point of the exponential Bellman operator. 
\begin{theorem} \label{thm:erm-main-convergence}
Whenever $\bm{w}^{\infty,\star } > -\bm{\infty}$ there exists $\pi\opt = (\bm{d}\opt)_{\infty} \in \PiSD$ such that
  \[
    \bm{w}^{\infty, \star }
    \; =\; 
    \bm{w}^{\infty}(\pi\opt)
    \; =\; 
    L^{\bm{d}\opt} \bm{w}^{\infty, \star },
  \]
and $\bm{w}^{\infty, \star }$ is the unique value that satisfies this equation.
\end{theorem}
\begin{corollary} \label{cor:erm-optimal-stationary}
Asuming the hypothesis of \cref{thm:erm-main-convergence}, we have that \(   \bm{v}^{\infty, \star} =  \bm{v}^{\infty}(\pi\opt) \) and
  \[
  g_{\infty}\opt(\beta) = \ermo^{\bm{\mu}}_{\beta}\left[ v^{\infty,\star}_{\tilde{s}_0} \right]
                   = \max_{\pi\in \PiSD} \ermo^{\bm{\mu}}_{\beta}\left[ v^{\infty}_{\tilde{s}_0} (\pi) \right].
  \]
\end{corollary}

We now outline the proof of \cref{thm:erm-main-convergence}; see \citet[appendix D.4]{su2024stationary} for details. To establish \cref{thm:erm-main-convergence}, we show that $\bm{w}^{t,\star }$ converges to a fixed point as $t\to \infty$. Standard arguments do not apply to our setting~\cite{puterman205markov,Kallenberg2021markov,patek2001terminating} because the ERM-TRC Bellman operator is not an $L_{\infty}$-contraction, it is not linear, and the values in value iteration do not increase or decrease monotonically. Although the exponential Bellman operator $L^{\bm{d}}$ is linear, it may not be a contraction. 

The main idea of the proof is to show that whenever the exponential value functions are bounded, the exponential Bellman operator must be \emph{weighted-norm} contraction with a unique fixed point. To facilitate the analysis, we define $\bm{w}^t\colon \PiSR^t \times \Real^{S} \to \Real^{S}, t \in \mathbb{N}$ for $\bm{z}\in \Real^{S},\,\pi \in \PiSR^t$, as
\begin{equation} \label{eq:exp-value-expression}
\begin{aligned}
      \bm{w}^t(\pi, \bm{z})
    \; =\; & \BellIE^{\bm{d}} \bm{w}(\pi_{1{:}t{-}1}) 
     \; =\; \BellIE^{\bm{d}} \BellIE^{\bm{d}} \dots \BellIE^{\bm{d}} (-\bm{z}) \\
    \; =\;& - (\bm{B}^{\bm{d}})^t \bm{\bm{z}} - \sum_{k=0}^{t-1}  (\bm{B}^{\bm{d}})^k \bm{b}^{\bm{d}} .
\end{aligned}
\end{equation}
The value $\bm{z}$ can be interpreted as the exponential value function at the termination of the process following $\pi$ for $t$ periods.
Note that \( \bm{w}^t(\pi) = \bm{w}^t(\pi, \bm{1}), \, \forall \pi \in \PiMR, t \in \mathbb{N} . \) 

An important technical result we show is that the only way a \emph{stationary} policy's return can be bounded is if the policy's matrix has a spectral radius strictly less than $1$.
\begin{lemma} \label{lem:bounded-contraction-fixedpoint}
For each $\pi = (\bm{d})_{\infty} \in \PiSR$ and $\bm{z} \ge \bm{0}$:
\[
  \bm{w}^{\infty}(\pi, \bm{z}) > -\bm{\infty}
  \quad \Rightarrow \quad
  \rho(\bm{B}^{\bm{d}}) < 1.
\]
\end{lemma}

\cref{lem:bounded-contraction-fixedpoint} uses the transience property to show that the Perron vector (with the maximum absolute eigenvalue) $\bm{f}$ of $\bm{B}^{\bm{d}}$ satisfies that $\bm{f}\tr \bm{b}^{\bm{d}}> 0$. Therefore, $\rho(\bm{B}^{\bm{d}}) < 1$ is necessary for the series in~\eqref{eq:exp-value-expression} to be bounded.

The limitation of \cref{lem:bounded-contraction-fixedpoint} is that it only applies to stationary policies. The lemma does not preclude the possibility that all stationary policies have unbounded returns, but a Markov policy with a bounded return exists. We construct an upper bound on $\bm{w}^{t,\star}$ that decreases monotonically in $t$ and converges to show this is impossible. The proof then concludes by squeezing $\bm{w}^{t,\star}$ between a lower and the upper bound with the same limits. This technique allows us to relax the limiting assumptions from prior work~\cite{patek2001terminating,de2020risk}. Finally, our results imply an optimal stationary policy exists whenever the planning horizon $T$ is sufficiently large. Because the set $\PiSD$ is finite, one policy must be optimal for a sufficiently large $T$. This property suggests behavior similar to \emph{turnpikes} in discounted MDPs~\cite{puterman205markov}.

\subsection{Algorithms}
We now briefly describe the algorithms we use to compute the optimal ERM-TRC policies. Surprisingly, the main algorithms for discounted MDPs, including value iteration, policy iteration, and linear programming, can be adapted to this risk-averse setting with only minor modifications.

\emph{Value iteration} is the most direct method for computing the optimal value function~\cite{puterman205markov}. The value iteration computes a sequence of $\bm{w}^k, k = 0, \dots $ such that
\[
 \bm{w}^{k+1} = L\opt \bm{w}^k, \quad \bm{w}^0 = \bm{0}. 
\]
The initialization of $\bm{w}^0 = \bm{0}$ is essential and guarantees convergence directly from the monotonicity argument used to prove \cref{thm:erm-main-convergence}.

\emph{Policy iteration} (PI) starts by initializing with a stationary policy $\pi_0 = (\bm{d}^0)_{\infty} \in \PiSD$. Then, for each iteration $k = 0, \dots $, PI alternates between the policy evaluation step and the policy improvement step:
\begin{equation*}
\bm{w}^k = - (\bm{I} - \bm{B}^{\bm{d}^k})^{-1} \bm{b}^{\bm{d}^k}, \;
  \bm{d}^{k+1} \in \argmax_{\bm{d}\in \mathcal{D}} \bm{B}^{\bm{d}} \bm{w}^k - \bm{b}^{\bm{d}}. 
\end{equation*}
PI converges because it monotonically improves the value functions when initialized with a policy $\bm{d}^0$ with bounded return~\cite{patek2001terminating}. However, we lack a practical approach to finding such an initial policy.

Finally, \emph{linear programming} is a fast and convenient method for computing optimal exponential value functions:
\begin{equation} \label{eq:erm-lp}
  \min \left\{  \bm{1}\tr \bm{w} \mid \bm{w}\in \Real^{S}, \bm{w} \ge  - \bm{b}^a  + \bm{B}^{a} \bm{w} ,    \,\forall a\in \actions \right\}.
\end{equation}
Here, $\bm{B}_{s,\cdot }^{a}  =(B_{s,s_1 }^{a}, \cdots, B_{s,s_{S}}^{a})$, $B_{s,s'}^{a}$ and $b_{s}^a$ are constructed as in~\eqref{eq:exponential-definitions}. We use the shorthand $\bm{B}^a = \bm{B}^{\bm{d}}$ and $\bm{b}^a = \bm{b}^{\bm{d}}$ where  $d_{a'}(s) = 1$ if $a = a'$ for each $s\in \mathcal{S}, a'\in \mathcal{A}$.

It is important to note that the value functions, as well as the coefficients of $\bm{B}^{\bm{d}}$ may be irrational. It is, therefore, essential to study the sensitivity of the algorithms to errors in the input. However, this question is beyond the scope of the present paper, and we leave it for future work. 

\section{Solving EVaR Total Reward Criterion}
\label{sec:ssp-with-evar}

This section shows that the EVaR-TRC objective can be reduced to a sequence of ERM-TRC problems, similarly to the discounted case~\cite{hau2023entropic}. As a result, an optimal stationary EVaR-TRC policy exists and can be computed using the methods described in \cref{sec:ssp-with-erm}.

Formally, we aim to compute a policy that maximizes the EVaR of the random return at some given fixed risk level $\alpha \in (0,1)$ defined as
\begin{equation} \label{eq:risk_object}
\sup_{\pi \in \PiHR} \liminf_{t\to \infty}  
\evaro_{\alpha}^{\pi, \bm{\mu}}
\left[ \sum_{k=0}^{t-1} r(\tilde{s}_k,\tilde{a}_k,\tilde{s}_{k+1}) \right].
\end{equation}
In contrast with \citet{ahmadi2021constrained}, the objective in~\eqref{eq:risk_object} optimizes EVaR rather than Nested EVaR.

\subsection{Reduction to ERM-TRC}
To solve~\eqref{eq:risk_object}, we exploit that EVaR can be defined in terms of ERM as shown in~\eqref{eq:evar-def-app}. To that end, define a function $h_t\colon \PiHR \times \Real \rightarrow \bar{\Real}$ for $t \in \mathbb{N} $ as
\begin{equation} \label{eq:h}
h_t(\pi, \beta) \;:=\;   g_t(\pi, \beta) + \beta^{-1} \log(\alpha),
\end{equation}
where $g_t$ is the ERM value of the policy defined in~\eqref{eq:g-definitions}. Also, $h_t\opt$, $h_{\infty}$, $h_{\infty}\opt$ are defined analogously in terms of $g_t\opt$, $g_{\infty}$, and $g_{\infty}\opt$ respectively. The functions $h$ are useful, because by~\eqref{eq:evar-def-app}:
\begin{equation} \label{eq:h-function-utility}
\evaro_{\alpha}^{\pi, \bm{\mu}}
\left[ \sum_{k=0}^{t-1} r(\tilde{s}_k,\tilde{a}_k,\tilde{s}_{k+1}) \right]
\; =\; 
\sup_{\beta > 0} h_t(\pi, \beta),
\end{equation}
for each $\pi\in \PiHR$ and $t \in \mathbb{N} $. However, note that the limit in the definition of $\sup_{\beta > 0} h_{\infty}\opt(\beta)$ is inside the supremum unlike in the objective in~\eqref{eq:risk_object}.

There are two challenges with solving~\eqref{eq:risk_object} by reducing it to~\eqref{eq:h-function-utility}. First, the supremum in the definition of EVaR in~\eqref{eq:evar-def-app} may not be attained, as mentioned previously. Second, the functions $g_t\opt$ and $h_t\opt$ may not converge \emph{uniformly} to $g_{\infty}\opt$ and $h_{\infty}\opt$. Note that \cref{thm:erm-main-convergence} only shows \emph{pointwise} convergence when the functions are bounded.

To circumvent the challenges described above, we replace the supremum in~\eqref{eq:h-function-utility} with a maximum over a \emph{finite} set $\mathcal{B}(\beta_0, \delta)$ of discretized $\beta$ values:
\begin{subequations} \label{eq:b-set-defs}
\begin{equation} \label{eq:b-set-definition}
  \mathcal{B}(\beta_0, \delta) \;:=\;  \left\{ \beta_0, \beta_1, \dots , \beta_K \right\},
\end{equation}
where $\delta > 0$,  $0 < \beta_0 < \beta_1 < \dots  < \beta_K$, and
\begin{equation}
\label{eq:beta-construction}
\beta_{k+1} \; :=\;  \frac{\beta_k \log \frac{1}{\alpha}}{\log\frac{1}{ \alpha } -\beta_k \delta} ,
\qquad
\beta_K \; \ge\;  \frac{\log \frac{1}{\alpha}}{\delta},
\end{equation}
\end{subequations}
for an appropriately chosen value $K$ for each $\beta_0$ and $\delta$. We assume that the denominator in the expression for $\beta_{k+1}$ in \cref{eq:beta-construction} is positive; otherwise $\beta_{k+1} = \infty$ and $\beta_k$ is sufficiently large.

The construction in~\eqref{eq:b-set-defs} resembles equations (19) and (20) in~\citet{hau2023entropic} but differs in the choice of $\beta_0$ because Hoeffding's lemma does not readily bound the TRC criterion.

The following proposition upper-bounds the value of $K$; see~\citep[theorem 4.3]{hau2023entropic} for a proof that $K$ is polynomial in $\delta$. 
\begin{proposition} \label{prop:choice-K}
  Assume a given $\beta_0 > 0$ and $\delta \in (0,1)$ such that $\beta_0 \delta < \log \frac{1}{\alpha}$. Then, to satisfy the condition in~\eqref{eq:beta-construction}, it is sufficient to choose $K$ as
  \begin{equation} \label{eq:K-inequality}
   K := \frac{\log  z}{\log (1-z)}, \quad\text{where}\quad z := \frac{\beta_0 \delta }{\log \frac{1}{\alpha}}. 
\end{equation}
\end{proposition}

The following theorem shows that one can obtain an optimal ERM policy for an appropriately chosen $\beta$ that approximates an optimal EVaR policy arbitrarily closely. 
\begin{theorem} \label{thm:optimal-evar-erm}
For any $\delta > 0$, let
\[
  (\pi\opt,\beta\opt)  \in
  \argmax_{(\pi,\beta)\in \PiSD \times \mathcal{B}(\beta_0, \delta)}
  h_{\infty}(\pi,\beta),
\]
where $\beta_0 > 0$ is chosen such that $g_{\infty}\opt(0) \le g_{\infty}\opt(\beta_0) - \delta$. Then the limits below exist and satisfy:
\begin{equation} \label{eq:evar-guarantee}
\lim_{t\to \infty}  
\evaro_{\alpha}^{\pi\opt, \bm{\mu}}
\left[ \sum_{k=0}^{t-1} r(\tilde{s}_k,\tilde{a}_k,\tilde{s}_{k+1}) \right]
  \ge
  \sup_{\pi\in \PiHR} \lim_{t\to \infty} \sup_{\beta>0} h(\pi,\beta)
  -\delta.
\end{equation}
\end{theorem}
Note that the right-hand side in~\eqref{eq:evar-guarantee} is the $\delta$-optimal objective in~\eqref{eq:risk_object}. 

The first implication of \cref{thm:optimal-evar-erm} is that there exists an optimal stationary deterministic policy. 
\begin{corollary} \label{cor:evar-stationary-policy}
  There exists an optimal stationary deterministic policy $\pi\opt \in \PiSD$ that attains the supremum in~\eqref{eq:risk_object}.
\end{corollary}

The second implication of \cref{thm:optimal-evar-erm} is that it suggests an algorithm for computing the optimal, or near-optimal, stationary policy. We summarize it in \cref{sec:evar-algorithms}.

\subsection{Algorithms}
\label{sec:evar-algorithms}

We now propose a simple algorithm for computing a $\delta$-optimal EVaR policy described in \cref{alg:evar-algo}. The algorithm reduces finding optimal EVaR-TRC policies to solving a sequence of ERM-TRC
problems in~\eqref{eq:sup-erm}. As \cref{thm:optimal-evar-erm} shows, there exists a $\delta$-optimal policy such that it is ERM-TRC optimal for some $\beta \in \mathcal{B}(\beta_0, \delta)$. It is, therefore, sufficient to compute an ERM-TRC  optimal policy for one of those $\beta$ values. 
\begin{algorithm}
\KwData{MDP and desired precision $\delta > 0$ }
\KwResult{$\delta$-optimal policy $\pi\opt \in \PiSD$}
\While{$g_{\infty}\opt(0) - g_{\infty}\opt(\beta_0) > \delta$}{
$\beta_0 \gets \beta_0 / 2$ \;
}
Construct $\mathcal{B}(\beta_0, \delta)$ as described in~\eqref{eq:b-set-definition} \;
Compute $\pi\opt  \in
\argmax_{\pi\in \PiSD} \max_{\beta\in \mathcal{B}(\beta_0, \delta)} h_{\infty}(\pi,\beta)$ by solving a linear program in~\eqref{eq:erm-lp} \;
\caption{Simple EVaR algorithm}
\label{alg:evar-algo}
\end{algorithm}

The analysis above shows that \cref{alg:evar-algo} is correct.
\begin{corollary} \label{prop:evar-algo-ok}
\Cref{alg:evar-algo} computes the $\delta$-optimal policy $\pi\opt \in \PiSD$ that satifies the condition~\eqref{eq:evar-guarantee}.
\end{corollary}
\Cref{prop:evar-algo-ok} follows directly from \cref{thm:optimal-evar-erm} and from the existence of a sufficiently small $\beta_0$ from the continuity of $g_{\infty}\opt(\beta)$ for positive $\beta$ around $0$.

\Cref{alg:evar-algo} prioritizes simplicity over computational complexity and could be accelerated significantly. Evaluating each $h_{\infty}\opt(\beta)$ requires computing an optimal ERM-TRC solution which involves solving a linear program. One could reduce the number of evaluations of $h_{\infty}\opt$ needed by employing a branch-and-bound strategy that takes advantage of the monotonicity of $g_{\infty}\opt$.

An additional advantage of \cref{alg:evar-algo} is that the overhead of computing optimal solutions for multiple risk levels $\alpha$ can be small if one selects an appropriate set $\mathcal{B}$.

\section{Numerical Evaluation}
\label{sec:numerical-eval}

In this section, we illustrate our algorithms and formulations on tabular MDPs that include positive and negative rewards. 

The ERM returns for the discounted and transient MDPs in \cref{fig:discounted-transient-mdp} with parameters $r = -0.2,\, \gamma = 0.9,\, \epsilon = 0.9$ are shown in \cref{fig:unbounded-erm}. The figure shows that, as expected, the returns are identical in the risk-neutral objective (when $\beta = 0$). However, for $\beta > 0$, the discounted and TRC returns differ significantly. The discounted return is unaffected by $\beta$ while the ERM-TRC return decreases with an increasing $\beta$. Please see \citet[appendix B]{su2024stationary} for more details.
 
\begin{figure}
\centering
\includegraphics[width=0.4\linewidth]{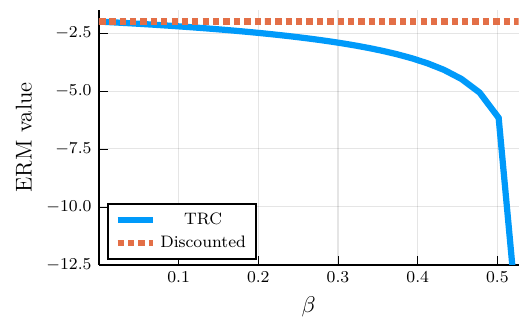} 
\caption{ERM values with TRC and discounted criteria.}
\label{fig:unbounded-erm}
\end{figure}

To evaluate the effect of risk-aversion on the structure of the optimal policy, we use the \emph{gambler's ruin} problem~\cite{hau2023entropic,bauerle2011markov}. In this problem, a gambler starts with a given amount of capital and seeks to increase it up to a cap $K$. In each turn, the gambler decides how much capital to bet. The bet doubles or is lost with a probability $q$ and $1-q$, respectively. The gambler can quit and keep the current wealth; the game also ends when the gambler goes broke or achieves the cap $K$. The reward equals the final capital, except it is -1 when the gambler is broke. The initial state is chosen uniformly. In the formulation, we use $q = 0.68$, and a cap is $K = 7$. The algorithm was implemented in Julia 1.10, and is available at \url{https://github.com/suxh2019/ERMLP}. Please see \citet[appendix F]{su2024stationary} for more details. 

\begin{figure}
\centering
\includegraphics[width=0.4\linewidth]{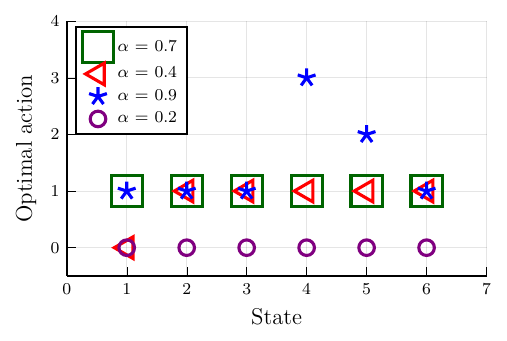}
\caption{The optimal EVaR-TRC policies.}
\label{fig:opt-policy}
\end{figure}

\Cref{fig:opt-policy} shows optimal policies for four different EVaR risk levels $\alpha$ computed by  \cref{alg:evar-algo}. The state represents how much capital the gambler holds. The optimal action indicates the amount of capital invested. The action 0 means quitting the game. Note that there is only one action when the capital is $0$ and $7$ for all policies so that action is neglected in \Cref{fig:opt-policy}. Because the optimal policy is stationary, we can interpret and analyze it. The policies become notably less risk-averse as $\alpha$ increases. For example, when $\alpha = 0.2$, the gambler is very risk-averse and always quits with the current capital. When $\alpha = 0.4$, the gambler invests $1$ when capital is greater than $1$ and quits otherwise to avoid losing it all. When $\alpha = 0.9$, the gambler makes bigger bets, increasing the probability of reaching the cap and losing all capital. 
 
\begin{figure}
\centering
\includegraphics[width=0.4\linewidth]{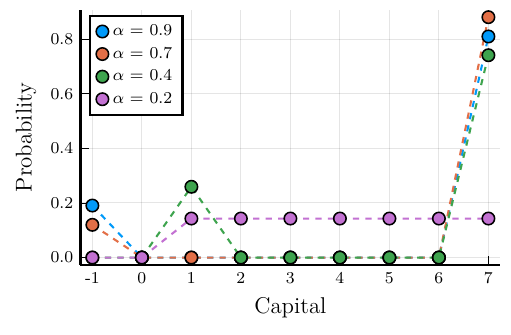}
\caption{Distribution of the final capital for EVaR optimal policies. }
\label{fig:capital-p0.65}
\end{figure}

To understand the impact of risk-aversion on the distribution of returns, we simulate the resulting policies over 7,000 episodes and show the distribution of capitals in \cref{fig:capital-p0.65}.
When $\alpha = 0.2$, the return follows a uniform distribution on [1, 7]. When $\alpha = 0.4$, the returns are $1$ and $7$.  When $\alpha = 0.7$ or $0.9$, the returns are $-1$ and $7$. Overall, the figure shows that for lower values of $\alpha$, the gambler gives up some probability of reaching the cap in exchange for a lower probability of losing all capital.

\section{Conclusion and Future Work}

We analyze transient MDPs with two risk measures: ERM and EVaR. We establish the existence of stationary deterministic optimal policies without any assumptions on the sign of the rewards, a significant departure from past work. Our results also provide algorithms based on value iteration, policy iteration, and linear programming for computing optimal policies.

Future directions include extensions to infinite-state TRC problems, risk-averse MDPs with average rewards, and partial-state observations.

\section*{Acknowledgments} 
We thank the anonymous reviewers for their detailed reviews and thoughtful comments, which significantly improved the paper's clarity. This work was supported, in part, by NSF grants 2144601 and 2218063. Julien Grand-Cl{\'e}ment was supported by Hi! Paris and Agence Nationale de la Recherche (Grant 11-LABX-0047).

\bibliography{aaai25}

\begin{thebibliography}{48}
\providecommand{\natexlab}[1]{#1}
\providecommand{\url}[1]{\texttt{#1}}
\expandafter\ifx\csname urlstyle\endcsname\relax
  \providecommand{\doi}[1]{doi: #1}\else
  \providecommand{\doi}{doi: \begingroup \urlstyle{rm}\Url}\fi

\bibitem[Ahmadi et~al.(2021{\natexlab{a}})Ahmadi, Dixit, Burdick, and
  Ames]{ahmadi2021risk}
Mohamadreza Ahmadi, Anushri Dixit, Joel~W Burdick, and Aaron~D Ames.
\newblock Risk-averse stochastic shortest path planning.
\newblock In \emph{IEEE Conference on Decision and Control (CDC)}, pages
  5199--5204, 2021{\natexlab{a}}.

\bibitem[Ahmadi et~al.(2021{\natexlab{b}})Ahmadi, Rosolia, Ingham, Murray, and
  Ames]{ahmadi2021constrained}
Mohamadreza Ahmadi, Ugo Rosolia, Michel~D Ingham, Richard~M Murray, and Aaron~D
  Ames.
\newblock Constrained risk-averse {Markov} decision processes.
\newblock In \emph{AAAI Conference on Artificial Intelligence}, volume~35,
  pages 11718--11725, 2021{\natexlab{b}}.

\bibitem[{Ahmadi-Javid}(2012)]{Ahmadi-Javid2012}
A.~{Ahmadi-Javid}.
\newblock Entropic {{Value-at-Risk}}: {{A}} new coherent risk measure.
\newblock \emph{Journal of Optimization Theory and Applications}, 155\penalty0
  (3):\penalty0 1105--1123, 2012.

\bibitem[Ahmadi-Javid and Pichler(2017)]{ahmadi2017analytical}
Amir Ahmadi-Javid and Alois Pichler.
\newblock An analytical study of norms and banach spaces induced by the
  entropic value-at-risk.
\newblock \emph{Mathematics and Financial Economics}, 11\penalty0 (4):\penalty0
  527--550, 2017.

\bibitem[Altman(1998)]{Altman1998}
Eitan Altman.
\newblock \emph{Constrained {Markov} Decision Processes}.
\newblock Routledge, 1998.

\bibitem[B{\"a}uerle and Glauner(2022)]{bauerle2022markov}
Nicole B{\"a}uerle and Alexander Glauner.
\newblock {Markov} decision processes with recursive risk measures.
\newblock \emph{European Journal of Operational Research}, 296\penalty0
  (3):\penalty0 953--966, 2022.

\bibitem[B{\"a}uerle and Ott(2011)]{bauerle2011markov}
Nicole B{\"a}uerle and Jonathan Ott.
\newblock {Markov} decision processes with average-value-at-risk criteria.
\newblock \emph{Mathematical Methods of Operations Research}, 74:\penalty0
  361--379, 2011.

\bibitem[Bertsekas and Tsitsiklis(1991)]{bertsekas1991analysis}
Dimitri~P Bertsekas and John~N Tsitsiklis.
\newblock An analysis of stochastic shortest path problems.
\newblock \emph{Mathematics of Operations Research}, 16\penalty0 (3):\penalty0
  580--595, 1991.

\bibitem[Bertsekas and Yu(2013)]{bertsekas2013stochastic}
Dimitri~P Bertsekas and Huizhen Yu.
\newblock Stochastic shortest path problems under weak conditions.
\newblock \emph{Lab. for Information and Decision Systems Report LIDS-P-2909,
  MIT}, 2013.

\bibitem[Blackwell(1967)]{blackwell1967positive}
David Blackwell.
\newblock Positive dynamic programming.
\newblock In \emph{Berkeley symposium on Mathematical Statistics and
  Probability}, volume~1, pages 415--418. University of California Press
  Berkeley, 1967.

\bibitem[Carpin et~al.(2016)Carpin, Chow, and Pavone]{carpin2016risk}
Stefano Carpin, Yin-Lam Chow, and Marco Pavone.
\newblock Risk aversion in finite {Markov} decision processes using total cost
  criteria and average value at risk.
\newblock In \emph{IEEE International Conference on Robotics and Automation
  (ICRA)}, pages 335--342, 2016.

\bibitem[Chung and Sobel(1987)]{Chung1987}
Kun-Jen Chung and Matthew~J. Sobel.
\newblock Discounted {{MDP}}'s: Distribution functions and exponential utility
  maximization.
\newblock \emph{SIAM Journal on Control and Optimization}, 25\penalty0
  (1):\penalty0 49--62, 1987.

\bibitem[Cohen et~al.(2021)Cohen, Efroni, Mansour, and
  Rosenberg]{cohen2021minimax}
Alon Cohen, Yonathan Efroni, Yishay Mansour, and Aviv Rosenberg.
\newblock Minimax regret for stochastic shortest path.
\newblock \emph{Advances in neural information processing systems},
  34:\penalty0 28350--28361, 2021.

\bibitem[Dann et~al.(2023)Dann, Wei, and Zimmert]{Dann2023}
Christoph Dann, Chen-Yu Wei, and Julian Zimmert.
\newblock A unified algorithm for stochastic path problems.
\newblock In \emph{International Conference on Learning Theory}, 2023.

\bibitem[de~Freitas et~al.(2020)de~Freitas, Freire, and Delgado]{de2020risk}
Elthon~Manhas de~Freitas, Valdinei Freire, and Karina~Valdivia Delgado.
\newblock Risk sensitive stochastic shortest path and logsumexp: From theory to
  practice.
\newblock In \emph{Intelligent Systems: Brazilian Conference (BRACIS)}, pages
  123--139. Springer, 2020.

\bibitem[{de Freitas} et~al.(2020){de Freitas}, Freire, and
  Delgado]{deFreitas2020}
Elthon~Manhas {de Freitas}, Valdinei Freire, and Karina~Valdivia Delgado.
\newblock Risk sensitive stochastic shortest path and logsumexp: From theory to
  practice.
\newblock In Ricardo Cerri and Ronaldo~C. Prati, editors, \emph{Intelligent
  {{Systems}}}, Lecture Notes in Computer Science, pages 123--139, 2020.

\bibitem[Delage et~al.(2019)Delage, Kuhn, and Wiesemann]{Delage2019}
Erick Delage, Daniel Kuhn, and Wolfram Wiesemann.
\newblock ``{{Dice}}''-sion-making under uncertainty: When can a random
  decision reduce risk?
\newblock \emph{Management Science}, 65\penalty0 (7):\penalty0 3282--3301,
  2019.

\bibitem[Denardo and Rothblum(1979)]{denardo1979optimal}
Eric~V Denardo and Uriel~G Rothblum.
\newblock Optimal stopping, exponential utility, and linear programming.
\newblock \emph{Mathematical Programming}, 16\penalty0 (1):\penalty0 228--244,
  1979.

\bibitem[Fei et~al.(2021{\natexlab{a}})Fei, Yang, Chen, and
  Wang]{fei2021exponential}
Yingjie Fei, Zhuoran Yang, Yudong Chen, and Zhaoran Wang.
\newblock Exponential bellman equation and improved regret bounds for
  risk-sensitive reinforcement learning.
\newblock \emph{Advances in Neural Information Processing Systems},
  34:\penalty0 20436--20446, 2021{\natexlab{a}}.

\bibitem[Fei et~al.(2021{\natexlab{b}})Fei, Yang, and Wang]{fei2021risk}
Yingjie Fei, Zhuoran Yang, and Zhaoran Wang.
\newblock Risk-sensitive reinforcement learning with function approximation: A
  debiasing approach.
\newblock In \emph{International Conference on Machine Learning}, pages
  3198--3207. PMLR, 2021{\natexlab{b}}.

\bibitem[Filar and Vrieze(2012)]{filar2012competitive}
Jerzy Filar and Koos Vrieze.
\newblock \emph{Competitive {Markov} decision processes}.
\newblock Springer Science \& Business Media, 2012.

\bibitem[Follmer and Schied(2016)]{Follmer2016stochastic}
Hans Follmer and Alexander Schied.
\newblock \emph{Stochastic finance: an introduction in discrete time}.
\newblock {De Gruyter Graduate}, 4th edition, 2016.

\bibitem[Freire and Delgado(2016)]{freire2016extreme}
Valdinei Freire and Karina~Valdivia Delgado.
\newblock Extreme risk averse policy for goal-directed risk-sensitive {Markov}
  decision process.
\newblock In \emph{Brazilian Conference on Intelligent Systems (BRACIS)}, pages
  79--84, 2016.

\bibitem[Gavriel et~al.(2012)Gavriel, Hanasusanto, and Kuhn]{gavriel2012risk}
Christos Gavriel, Grani Hanasusanto, and Daniel Kuhn.
\newblock Risk-averse shortest path problems.
\newblock In \emph{IEEE Conference on Decision and Control (CDC)}, pages
  2533--2538, 2012.

\bibitem[{Grand-Cl{\'e}ment} and Petrik(2022)]{Grand-Clement2022}
Julien {Grand-Cl{\'e}ment} and Marek Petrik.
\newblock Towards convex optimization formulations for robust {{MDPs}}, 2022.

\bibitem[Hau et~al.(2023{\natexlab{a}})Hau, Delage, Ghavamzadeh, and
  Petrik]{Hau2023a}
Jia~Lin Hau, Erick Delage, Mohammad Ghavamzadeh, and Marek Petrik.
\newblock On dynamic programming decompositions of static risk measures in
  {Markov} decision processes.
\newblock In \emph{Neural {{Information Processing Systems}} ({{NeurIPS}})},
  2023{\natexlab{a}}.

\bibitem[Hau et~al.(2023{\natexlab{b}})Hau, Petrik, and
  Ghavamzadeh]{hau2023entropic}
Jia~Lin Hau, Marek Petrik, and Mohammad Ghavamzadeh.
\newblock Entropic risk optimization in discounted mdps.
\newblock In \emph{International Conference on Artificial Intelligence and
  Statistics}, pages 47--76. PMLR, 2023{\natexlab{b}}.

\bibitem[Horn and Johnson(2013)]{Horn2013}
Roger~A. Horn and Charles~A Johnson.
\newblock \emph{Matrix Analysis}.
\newblock {Cambridge University Press}, 2nd edition, 2013.

\bibitem[James and Collins(2006)]{james2006analysis}
Huw~W James and EJ~Collins.
\newblock An analysis of transient {Markov} decision processes.
\newblock \emph{Journal of applied probability}, 43\penalty0 (3):\penalty0
  603--621, 2006.

\bibitem[Johnsonbaugh and Pfaffenberger(1981)]{Johnsonbaugh1981}
Richard Johnsonbaugh and W.~E. Pfaffenberger.
\newblock \emph{Foundations of Mathematical Analysis}.
\newblock Dover Publications, 1981.

\bibitem[Kallenberg(2021)]{Kallenberg2021markov}
Lodewijk Kallenberg.
\newblock Markov decision processes.
\newblock \emph{Lecture Notes. University of Leiden}, 2021.

\bibitem[Kastner et~al.(2023)Kastner, Erdogdu, and Farahmand]{Kastner2023}
Tyler Kastner, Murat~A. Erdogdu, and Amir-massoud Farahmand.
\newblock Distributional model equivalence for risk-sensitive reinforcement
  learning.
\newblock In \emph{Conference on {{Neural Information Processing Systems}}},
  2023.

\bibitem[Kupper and Schachermayer(2006)]{Kupper2006}
Michael Kupper and Walter Schachermayer.
\newblock Representation results for law invariant time consistent functions.
\newblock \emph{Mathematics and Financial Economics}, 16\penalty0 (2):\penalty0
  419--441, 2006.

\bibitem[Lam et~al.(2022)Lam, Verma, Low, and Jaillet]{lam2022risk}
Thanh Lam, Arun Verma, Bryan Kian~Hsiang Low, and Patrick Jaillet.
\newblock Risk-aware reinforcement learning with coherent risk measures and
  non-linear function approximation.
\newblock In \emph{International Conference on Learning Representations
  (ICLR)}, 2022.

\bibitem[Li et~al.(2022)Li, Zhong, and Brandeau]{li2022quantile}
Xiaocheng Li, Huaiyang Zhong, and Margaret~L Brandeau.
\newblock Quantile {Markov} decision processes.
\newblock \emph{Operations research}, 70\penalty0 (3):\penalty0 1428--1447,
  2022.

\bibitem[Marthe et~al.(2023)Marthe, Garivier, and Vernade]{Marthe2023a}
Alexandre Marthe, Aur{\'e}lien Garivier, and Claire Vernade.
\newblock Beyond average return in {Markov} decision processes.
\newblock In \emph{Conference on {{Neural Information Processing Systems}}},
  2023.

\bibitem[Meggendorfer(2022)]{meggendorfer2022risk}
Tobias Meggendorfer.
\newblock Risk-aware stochastic shortest path.
\newblock In \emph{AAAI Conference on Artificial Intelligence}, volume~36,
  pages 9858--9867, 2022.

\bibitem[Patek(2001)]{patek2001terminating}
Stephen~D Patek.
\newblock On terminating {Markov} decision processes with a risk-averse
  objective function.
\newblock \emph{Automatica}, 37\penalty0 (9):\penalty0 1379--1386, 2001.

\bibitem[Patek and Bertsekas(1999)]{patek1999stochastic}
Stephen~D Patek and Dimitri~P Bertsekas.
\newblock Stochastic shortest path games.
\newblock \emph{SIAM Journal on Control and Optimization}, 37\penalty0
  (3):\penalty0 804--824, 1999.

\bibitem[Patek(1997)]{patek1997stochastic}
Stephen~David Patek.
\newblock \emph{Stochastic and shortest path games: theory and algorithms}.
\newblock PhD thesis, Massachusetts Institute of Technology, 1997.

\bibitem[Pflug and Pichler(2016)]{pflug2016time}
Georg~Ch Pflug and Alois Pichler.
\newblock Time-consistent decisions and temporal decomposition of coherent risk
  functionals.
\newblock \emph{Mathematics of Operations Research}, 41\penalty0 (2):\penalty0
  682--699, 2016.

\bibitem[Puterman(2005)]{puterman205markov}
Martin~L Puterman.
\newblock \emph{{Markov} decision processes: discrete stochastic dynamic
  programming}.
\newblock John Wiley \& Sons, 2005.

\bibitem[Smith and Chapman(2023)]{smith2023exponential}
Kevin~M Smith and Margaret~P Chapman.
\newblock On exponential utility and conditional value-at-risk as risk-averse
  performance criteria.
\newblock \emph{IEEE Transactions on Control Systems Technology}, 2023.

\bibitem[Su and Petrik(2023)]{su2023solving}
Xihong Su and Marek Petrik.
\newblock Solving multi-model mdps by coordinate ascent and dynamic
  programming.
\newblock In \emph{Uncertainty in Artificial Intelligence}, pages 2016--2025.
  PMLR, 2023.

\bibitem[Su et~al.(2024{\natexlab{a}})Su, Grand-Cl{\'e}ment, and
  Petrik]{su2024stationary}
Xihong Su, Julien Grand-Cl{\'e}ment, and Marek Petrik.
\newblock Risk-averse total-reward mdps with erm and evar.
\newblock \emph{arXiv preprint arXiv:2408.17286}, 2024{\natexlab{a}}.

\bibitem[Su et~al.(2024{\natexlab{b}})Su, Petrik, and
  Grand-Cl{\'e}ment]{su2024optimality}
Xihong Su, Marek Petrik, and Julien Grand-Cl{\'e}ment.
\newblock Optimality of stationary policies in risk-averse total-reward mdps
  with evar.
\newblock In \emph{ICML 2024 Workshop: Foundations of Reinforcement Learning
  and Control--Connections and Perspectives}, 2024{\natexlab{b}}.

\bibitem[Su et~al.(2024{\natexlab{c}})Su, Petrik, and
  Grand-Cl{\'e}ment]{suevar}
Xihong Su, Marek Petrik, and Julien Grand-Cl{\'e}ment.
\newblock Evar optimization in mdps with total reward criterion.
\newblock In \emph{Seventeenth European Workshop on Reinforcement Learning},
  2024{\natexlab{c}}.

\bibitem[Sutton and Barto(2018)]{Sutton2018}
R~S Sutton and A~G Barto.
\newblock \emph{Reinforcement Learning: An Introduction}.
\newblock The MIT Press, 2nd edition, 2018.

\end{thebibliography}

\appendix
\section{Background}

\begin{proposition}{Tower Property for ERM}
   \label{pro:erm-tower-property} 
   For any two random variables $\tilde{x}_1,\tilde{x}_2 \in \mathbb{X}$, we have
\[
\ermo_{\beta}[\tilde{x}_1] = \ermo_{\beta}[\ermo_{\beta}[\tilde{x}_1 \mid  \tilde{x}_2]],
\]
where the conditional ERM is defined as 
\[
\ermo_{\beta}[\tilde{x}_1 | \tilde{x}_2] = -\beta^{-1} \log(\E[e^{-\beta \tilde{x}_1} \mid \tilde{x}_2])
\]
\end{proposition}
For a proof, see, for example, \citet[theorem~3.1]{hau2023entropic}.

\section{Discounting versus TRC}
\label{sec:risk-neutral-disc-trc}
The discounted infinite-horizon ERM objective for a discount factor $\gamma\in (0,1)$, a policy $\pi\in \PiSD$, and $\beta > 0$ is 
\begin{equation} \label{eq:discounted-objective}
  \rho_{\gamma}(\pi, \beta) := \ermo^{\pi,\bm{\mu}} \left[\sum_{t=0}^{\infty} \gamma^t \cdot
    r(\tilde{s}_t, \tilde{a}_t, \tilde{s}_{t+1})\right].
\end{equation}

In the risk-neutral setting, it is well-known that the discounted objective can be interpreted as TRC~\citep[Section 1.10]{Altman1998}. We summarize the construction in this section and show that it does not extend to risk-averse settings.

Given an MDP $\mathcal{M} = (\states, \actions, p, r, \mu)$, we construct a transient MDP $\bar{\mathcal{M}}_{\gamma} = (\bar{\states}, \actions, \bar{p}, \bar{r}, \bar{\mu})$ such that $\bar{\mathcal{S}} = \mathcal{S} \cup \left\{ e \right\}$, and $\bar{\mu}(s) = \mu(s), \forall s\in \mathcal{S}$ and $\mu(e) = 0$. The transition function $\bar{p}$ is defined as
\begin{equation}
\label{eq:convert-p}
  \bar{p}(s,a,s') \;=\; 
  \begin{cases}
    \gamma \cdot  p(s,a,s') \quad&\text{if} \quad s,s'\in \mathcal{S}, a\in \mathcal{A}, \\
    1-\gamma \quad&\text{if} \quad s\in \mathcal{S}, s' = e, a\in \mathcal{A}, \\
    1 \quad&\text{if} \quad s = s' = e, a\in \mathcal{A}, \\
    0 \quad&\text{otherwise}.
  \end{cases}
\end{equation}
When the rewards $r\colon \mathcal{S} \times \mathcal{A} \to \Real$ are independent of the next state, then $\bar{r}\colon \bar{\mathcal{S}} \times \mathcal{A} \to \Real$ are defined as
\begin{equation}
\label{eq:convert-r}
      \bar{r}(s,a) =
  \begin{cases}
    r(s,a) &\text{if } s\in \mathcal{S}, \\
    0 &\text{otherwise}.
  \end{cases}
\end{equation}

The model can be readily extended to account for the target state dependence by constructing an $\bar{\mathcal{M}}_{\gamma}$ with the reward function $\bar{r}\colon \bar{\mathcal{S}} \times  \mathcal{A} \times \bar{\mathcal{S}} \to  \Real$ as 
\[
  \bar{r}(s,a,s') =
  \begin{cases}
    \gamma^{-1} \cdot r(s,a,s') &\text{if } s, s'\in \mathcal{S}, \\
    0 &\text{otherwise}.
  \end{cases}
\]

Recall that when $\beta = 0$, we have $\lim_{\beta \rightarrow 0^+} \ermo[\tilde{x}] = \E[\tilde{x}]$.

\cref{prop:discount-ssp} shows that the expected total returns are identical in a discounted MDP $\mathcal{M}$ and in a transient MDP $\bar{\mathcal{M}}_{\gamma}$.
\begin{proposition} \label{prop:discount-ssp}
For each $\gamma\in [0,1)$ and MDPs $\mathcal{M}, \bar{\mathcal{M}}_{\gamma}$ constructed above, we have that
\[
\rho_{\gamma}(\pi, 0) \; =\;
g_{\infty}( \pi, 0),
\quad \pi\in \PiSD.
\]
where $\rho_{\gamma}(\pi, 0)$ is defined in~\eqref{eq:discounted-objective} and $g_{\infty}( \pi, 0)$ is defined in~\eqref{eq:g-definitions}.

\end{proposition}
\begin{proof}
The proposition follows from the construction above by algebraic manipulation~\citep[section 1.10]{Altman1998}. 
\end{proof}

\begin{proposition}
There exists an MDP $\mathcal{M}$, $\gamma > 0$, $\beta > 0$, $\bar{\mathcal{M}}_{\gamma}$ constructed above, and $\pi\in \PiSD$ such that
\[
\rho_{\gamma}( \pi, \beta) \; \neq \;
g_{\infty}( \pi, \beta),
\]
where $\rho_{\gamma}(\pi, 0)$ is defined in~\eqref{eq:discounted-objective} and $g_{\infty}( \pi, 0)$ is defined in~\eqref{eq:g-definitions}.
\end{proposition}
\begin{proof}
See the one-state example described in \cref{sec:numerical-eval}.
\end{proof}

\section{Additional Results of \cref{sec:preliminaries} }

The following proposition contradicts theorem~1 in \citet{ahmadi2021risk}.
\begin{proposition}
\label{thm:cvar-counter-example}
There exists a transient MDP and a risk level $\alpha \in (0,1)$ such that the TRC with nested CVaR or EVaR is unbounded. 
\end{proposition}
\begin{proof}
Consider the example of a transient MDP in \cref{fig:discounted-transient-mdp} with $r=-1$, a risk level $\alpha \in (0,1)$, and some $\epsilon \in (0,1)$ such that $\epsilon \ge \alpha$. This MDP admits only a single policy because it only has one action $a$, which we omit in the notation below. 

Recall that nested CVaR objective for some $s \neq e$ is
\begin{equation*}
  v_t(s)
  \; =\;
  \cvaro^{s}_{\alpha}\left[r(\tilde{s}_0, a)
+\cvaro_{\alpha}^{\tilde{s}_1}\left[r(\tilde{s}_1, a) + \cdots 
\cvaro^{\tilde{s}_{t-1}}_{\alpha}\left[r(\tilde{s}_{t-1}, a) \right] \right] \right],
\end{equation*}

 The value function for the non-terminal state can be computed using a dynamic program as~\citep[theorem~1]{ahmadi2021risk}:
\begin{align*}
  v_t(s)
  &= \cvaro^{s}_{\alpha}[r(s,a) + v_{t-1}(\tilde{s}_1)] 
  = -1 + \min \left\{ q_1 \cdot v_{t-1}(s) \mid q\in \Delta^2,\, 
  q_1 \le \frac{\epsilon}{\alpha}, q_2 \le \frac{1-\epsilon}{\alpha} \right\} \\
  &\le -1 - v_{t-1}(s).
\end{align*}
Here, we used the dual representation of CVaR. Then, by induction, $v_t(s) \le -t$ and, therefore, $\lim_{t\to \infty} v_t(s) = -\infty$.
\end{proof}

\section{Proofs of  \cref{sec:ssp-with-erm}}

\begin{proposition} \label{prop:erm-unbounded}
There exists a transient MDP and a risk level $\beta > 0$ such that $g_{\infty}\opt(\beta) = -\infty$.
\end{proposition}
\begin{proof}
We use the transient MDP described in \cref{fig:discounted-transient-mdp} to show this result. Because the returns of this MDP follow a truncated geometric distribution, its risk-averse return for each $\beta>0$ and $\epsilon \in (0,1)$ can be expressed analytically for $t \ge 1$ as
\begin{equation} \label{eq:exp-erm-value-func} 
\begin{aligned}
 \ermp{\beta}{\pi}{\sum_{k=0}^{t-1} r(\tilde{s}_k,\tilde{a}_k,\tilde{s}_{k+1})} 
&= -\frac{1}{\beta}\log\left(  \sum_{k = 0}^{t-1} (1-\epsilon) \epsilon^k \cdot \exp{-\beta \cdot k \cdot  r}  + \epsilon^t \cdot 0 \right) \\
&=  -\frac{1}{\beta}\log\left(  \sum_{k = 0}^{t-1} (1-\epsilon) \epsilon^{k} \cdot \exp{-\beta \cdot  r}^k  \right).
\end{aligned}
\end{equation}
Here, $(1-\epsilon) \epsilon^k$ is the probability that the process terminates after exactly $k$ steps, and $\epsilon^t$ is the probability that the process does not terminate before reaching the horizon. Then, using the fact that a geometric series $\sum_{i=0}^{\infty} a \cdot q^i$ for $a \neq 0$ is bounded if and only if $|q| < 1$ we get that
\begin{equation*}
\lim_{t\to \infty} \ermp{\beta}{\pi}{\sum_{k=0}^{t-1} r(\tilde{s}_k,\tilde{a}_k,\tilde{s}_{k+1})}   > - \infty 
\qquad\Leftrightarrow \qquad
\epsilon \cdot \exp{-\beta \cdot  r}  < 1. 
\end{equation*}

Note that $\epsilon \cdot \exp{-\beta \cdot r} \ge 0$ from its definition. Then, setting $r = -1$ and $\beta > -\log \epsilon$ proves the result. 
\end{proof}

\subsection{Optimality of Markov Policies} \label{sec:optim-mark-polic}

The equivalence to solving finite-horizon MDPs with exponential utility functions gives us the following result.
\begin{theorem} \label{thm:optim-mark-polic}
For each $\beta > 0$, there exists an optimal deterministic Markov policy $\pi^{t,\star} \in \PiMD$ for each horizon $t \in \mathbb{N} $: 
\begin{equation*}
\begin{split}
 \max_{\pi\in \PiMD}  \ermp{\beta}{\pi,s}{\sum_{k=0}^{t-1} r(\tilde{s}_k,\tilde{a}_k,\tilde{s}_{k+1})}
  \;=\;   \max_{\pi \in \PiHR} \ermp{\beta}{\pi,s}{\sum_{k=0}^{t-1} r(\tilde{s}_k,\tilde{a}_k,\tilde{s}_{k+1})}.
  \end{split}
\end{equation*}
\end{theorem}
See \citep[Corollary~4.2]{hau2023entropic} for a proof. The result can also be derived from the optimality of Markov deterministic policies in MDPs with exponential utility functions~\cite{Chung1987, patek2001terminating}.

\subsection{Bellman Operator}

\begin{lemma}\label{lem:exp-bellman-monotone}
The exponential Bellman operator is monotone. That is, for $\bm{x},\bm{y} \in \Real^S$
\begin{align}
  \label{eq:exp-bellman-monotone}
  \bm{x} \ge \bm{y}
  &\quad \Rightarrow \quad
    L^{\bm{d}} \bm{x} \ge L^{\bm{d}} \bm{y}, \qquad \forall \bm{d}\in \mathcal{D} \\
  \label{eq:exp-bellman-opt-monotone}
  \bm{x} \ge \bm{y}
  &\quad \Rightarrow \quad
  L\opt \bm{x} \ge L\opt \bm{y}.
\end{align}
\end{lemma}
\begin{proof}
The property in~\eqref{eq:exp-bellman-monotone} follows immediately from non-negativity of $\bm{B}^{\bm{d}}$. The property in~\eqref{eq:exp-bellman-opt-monotone} then follows from the monotonicity of the $\max$ operator.
\end{proof}

\begin{lemma}\label{lem:exp-bellman-continuous}
The exponential Bellman operators $L^{\bm{d}},\,  \forall \bm{d}\in \mathcal{D}$ and $L\opt$ are continuous.
\end{lemma}
\begin{proof}
The lemma follows directly from the continuity of linear operators and from the fact that the pointwise maximum of a finite number of continuous functions is continuous. See also \citep[lemma~5]{patek2001terminating}
\end{proof}

\subsection{Proof of \cref{thm:exponential-bellman}}
\label{proof:thm-exponential-bellman}

\begin{proof}[Proof of \cref{thm:exponential-bellman}]
To construct the value function, we can define a Bellman operator $\BellI^{\bm{d}}\colon \Real^S \to  \Real^S$ for any decision rule $\bm{d}\colon  \mathcal{S} \to  \probs{A}$ and the optimal Bellman operator $\BellI\opt \colon \Real^S \to \Real^S$ for a \emph{value vector} $\bm{v}\in \Real^S$ as
\begin{equation} \label{eq:bellman-erm}
\begin{aligned}
  (\BellI^{\bm{d}} \bm{v})_s
  &\;:=\; \ermp{\beta}{\bm{d},s}{ r(s,\tilde{a}_0, \tilde{s}_1) + v_{\tilde{s}_1} },  \\
  \BellI\opt \bm{v}
  &\;:=\; \max_{\bm{d}\in \mathcal{D}}   \BellI^{\bm{d}} \bm{v} = \max_{\bm{d}\in \ext \mathcal{D}}   \BellI^{\bm{d}} \bm{v}.
\end{aligned}
\end{equation}
It is easy to see that $\bm{d}$ can be chosen independently for each state to maximize $\bm{v}$ uniformly across states. The optimality of deterministic decision rules, $\bm{d} \in \ext \mathcal{D}$, follows because ERM is a mixture quasi-convex function~\cite{Delage2019}.

The existence of value function for finite-horizon problem under the ERM objective has been analyzed previously~\cite{hau2023entropic} including in the context of exponential utility functions~\cite{Chung1987}.

To derive the exponential Bellman operator for the exponential value function for $\bm{d}\colon \mathcal{S}\to \probs{A}$, we concatenate the Bellman operator with the transformations to and from the exponential value function: 
\begin{equation} \label{eq:erm-exp-bellman}
\begin{aligned}
 (\BellIE^{\bm{d}} \bm{w})_s 
  &\;=\; - \exp {-\beta \cdot T^{\bm{d}} ( - \beta^{-1} \log (-  \bm{w}) ) }\\
  &\; =\;  -\E^{\bm{d},s}\left[\exp{-\beta \cdot  r(s,\tilde{a}_0,\tilde{s}_1)  + \log (-w_{\tilde{s}_1}) } \right] \\
  &\; =\;  -\E^{\bm{d},s}\left[\exp{-\beta \cdot  r(s,\tilde{a}_0,\tilde{s}_1)}  \cdot  (-w_{\tilde{s}_1})  \right] \\
  &\; =\; \sum_{s'\in \bar{\mathcal{S}}} \sum_{a\in \mathcal{A}}  p(s, a, s') \cdot d_a(s) \cdot \exp{-\beta \cdot  r(s,a,s')}  \cdot  w_{s'} \\
  &\; =\;  \sum_{s'\in \mathcal{S}} \sum_{a\in \mathcal{A}}  p(s, a, s') \cdot d_a(s) \cdot \exp{-\beta \cdot  r(s,a,s')}  \cdot  w_{s'} \\
   & \qquad \qquad-  \sum_{a\in \mathcal{A}}  p(s, a, e) \cdot d_a(s) \cdot \exp{-\beta \cdot  r(s,a,e)}.
\end{aligned}
\end{equation}

The derivation above uses the fact that $w_e = -1$ since $v_e = 0$ by definition. The statement of the theorem then follows by algebraic manipulation of $\bm{B}^{\bm{d}}, \bm{b}^{\bm{d}}$ and by induction on $t$. The base case hold by the definition of $\bm{w}^0(\pi) = \bm{w}^{0,\star } = -\bm{1}$.

The existence of an optimal $\pi\opt$ follows by choosing the maximum in the definition of $L\opt $, which is attained by compactness and continuity of the objective. 
\end{proof}

\subsection{Proof of \cref{thm:erm-main-convergence}}
\label{proof:erm-main-convergence}

\begin{lemma} \label{lem:radius-contractions}
Assume some $\pi = (\bm{d})_{\infty} \in \PiSR$ such that $\rho(\bm{B}^{\bm{d}}) < 1$. Then for all $\bm{z}\in \Real^S$
\[
   \bm{w}^{\infty}(\pi)  = \bm{w}^{\infty}(\pi, \bm{z}) = L^{\bm{d}} \bm{w}^{\infty}(\pi) > -\bm{\infty}.
  \]
\end{lemma}
\begin{proof}
The result follows by algebraic manipulation from~\eqref{eq:exp-value-expression} and basic matrix analysis. When $\rho(\bm{B}^{\bm{d}}) < 1$, we get from Neumann series~\citep[problem~5.6.P26]{Horn2013}
\[
 \lim_{t\to \infty} \sum_{k=0}^{t-1}  (\bm{B}^{\bm{d}})^k \bm{b}^{\bm{d}} =  (\bm{I} - \bm{B}^{\bm{d}})^{-1}\bm{b}^{\bm{d}},
\]
and a consequence of Gelfand's formula~\citep[theorem~4.5]{Kallenberg2021markov}
\[
  \lim_{k\to \infty} (\bm{B}^{\bm{d}})^k \bm{z} = \bm{0}.
\]
\end{proof}

\begin{proof}[Proof of \cref{thm:erm-main-convergence}]

For the remainder of the proof, let $\pi_{\mathrm{M}}\opt \in \arg\max_{\pi\in \PiMR} \bm{\mu}\tr \bm{v}^{\infty}(\pi)$ and suppose that $\bm{\mu}\tr \bm{v}^{\infty}(\pi\opt_{\mathrm{M}}) > -\infty$. Then, the exponential value function $\bm{w}^{t,\star}  = \bm{w}^t(\pi_{\mathrm{M}}\opt) \in \Real^S$ of $\pi_{\mathrm{M}}\opt $ satisfies by \cref{thm:exponential-bellman} that
\[
\bm{w}^{0,\star} = -\bm{1},
\qquad
\bm{w}^t(\pi\opt_{\mathrm{M}})  = L\opt \bm{w}^{t-1}(\pi\opt_{\mathrm{M}}), \quad t = 1, \dots .
\]

We show that $\lim_{t\to \infty} \bm{w}^{t,\star}$ exists and that it is attained by a stationary policy. We construct a sequence $\bm{w}_{\mathrm{u}}^t \in \Real^S, t \in \mathbb{N} $ as
\[
\bm{w}_{\mathrm{u}}^0 = \bm{0},
\qquad
\bm{w}_{\textrm{u}}^t  = L\opt \bm{w}_{\textrm{u}}^{t-1}, \quad t = 1, \dots .
\]
First, we show by induction that
\begin{equation} \label{eq:w-hat-upper-bound}
  \bm{w}_{\textrm{u}}^t \quad\ge\quad \bm{w}^{t,\star}, \qquad t \in \mathbb{N}.
\end{equation}
The base case $t=0$ follows immediately from the definitions of $\bm{w}_{\textrm{u}}^0$ and $\bm{w}^{0,\star }$. Next, suppose that~\eqref{eq:w-hat-upper-bound} holds for some $t > 0$, then it also holds for $t+1$:
\[
 \bm{w}_{\textrm{u}}^{t+1} \; =\;  L\opt \bm{w}_{\textrm{u}}^t \; \ge\;  L\opt \bm{w}^{t,\star} \; =\;   \bm{w}^{t+1,\star}, 
\]
where the inequality follows from the inductive assumption and from \cref{lem:exp-bellman-monotone}.
Second, we show by induction that
\begin{equation}\label{eq:w-hat-decreasing}
  \bm{w}_{\textrm{u}}^{t+1} \quad\le\quad \bm{w}_{\textrm{u}}^{t}, \qquad t \in \mathbb{N}.
\end{equation}
The base case for $t=0$ holds as
\[
  \bm{w}_{\textrm{u}}^1 \; =\;  L\opt \bm{w}_{\textrm{u}}^0 \; =\;  L\opt \bm{0} \; =\;  \max_{\bm{d}\in \mathcal{D}} - \bm{b}^{\bm{d}} \; \le\;  \bm{0} \; =\; \bm{w}_{\textrm{u}}^0, 
\]
where the inequality holds because $\bm{b}^{\bm{d}} \ge 0$ from its construction. To prove the inductive step, assume that~\eqref{eq:w-hat-decreasing} holds for $t>0$ and prove it for $t+1$:
\[
  \bm{w}_{\textrm{u}}^{t+1} \; =\;  L\opt \bm{w}_{\textrm{u}}^t \; \le\;  L\opt \bm{w}_{\textrm{u}}^{t-1} \; =\;  \bm{w}_{\textrm{u}}^t,
\]
where the inequality follows from the inductive assumption and from \cref{lem:exp-bellman-monotone}.

Then, using the Monotone Convergence Theorem~\cite[theorem~16.2]{Johnsonbaugh1981}, finite $\mathcal{S}$, and $\inf_{t \in \mathbb{N}}\bm{w}_{\textrm{u}}^t \ge \inf_{t \in \mathbb{N} } \bm{w}^{t, \star} > -\infty$, we get that there exists $\bm{w}_{\mathrm{u}}\opt  \in \Real^S$ such that
\[
\bm{w}_{\textrm{u}}\opt = \lim_{t\to \infty} \bm{w}_{\textrm{u}}^t,  
\]
and the limit exists. Then, taking the limit of both sides of $\bm{w}_{\textrm{u}}^t  = L\opt \bm{w}_{\textrm{u}}^{t-1}$, we have that
\begin{align*}
  \lim_{t\to \infty} \bm{w}_{\textrm{u}}^t  &= \lim_{t\to \infty} L\opt \bm{w}_{\textrm{u}}^{t-1} \\
  \bm{w}_{\textrm{u}}\opt   &= \lim_{t\to \infty} L\opt \bm{w}_{\textrm{u}}^{t-1} \\
  \bm{w}_{\textrm{u}}\opt   &=  L\opt \lim_{t\to \infty}\bm{w}_{\textrm{u}}^{t-1}  \\
  \bm{w}_{\mathrm{u}}\opt &= L^{\bm{d}\opt} \bm{w}_{\mathrm{u}}\opt ,
\end{align*}
where $\bm{d}\opt = \argmax_{\bm{d}\in \mathcal{D}} L^{\bm{d}} \bm{w}_{\mathrm{u}}$. Above, we can exchange the operators $L\opt $ and $\lim$ by the continuity of $L\opt $~(\cref{lem:exp-bellman-continuous}).

Now, define $\bm{w}_{\mathrm{l}}^t \in \Real^S, t \in \mathbb{N}$ as
\[
\bm{w}_{\mathrm{l}}^0 = -\bm{1},
\qquad
\bm{w}_{\textrm{l}}^t  = L^{\bm{d}\opt} \bm{w}_{\textrm{l}}^{t-1}, \quad t = 1, \dots .
\]
From the definition of $L\opt$ and by induction on $t$ we have that
\[
 \bm{w}_{\mathrm{l}}^t  \le \bm{w}^{t,\star}.
\]
By \cref{lem:bounded-contraction-fixedpoint} for $\bm{z} = \bm{0}$, we have that $\rho(\bm{B}^{\bm{d}\opt}) < 1$ and therefore from \cref{lem:radius-contractions}
\[
\lim_{t\to \infty} \bm{w}_{\mathrm{l}}^{t}  = \lim_{t\to \infty}  \bm{w}_{\mathrm{u}}^t = \bm{w}_{\mathrm{u}}\opt.
\]
In addition, because
\[
 \bm{w}_{\mathrm{u}}^t\; \ge \; \bm{w}^{t,\star }\; \ge \; \bm{w}_{\mathrm{l}}^t, \qquad t \in \mathbb{N},  
\]
the Squeeze Theorem~\cite[theorem~14.3]{Johnsonbaugh1981} shows that
\[
 \lim_{t \to \infty} \bm{w}^{t, \star } = \bm{w}_{\mathrm{u}}\opt , 
\]
and $\bm{d}\opt $ is a stationary policy that attains the return of $\pi\opt_{\mathrm{M}}$.
\end{proof}

\subsection{Proof of \cref{cor:erm-optimal-stationary}}
\label{proof:cor-erm-optimal-stationary}
\begin{proof}[Proof of \cref{cor:erm-optimal-stationary}]
  From the existence of an optimal stationary policy $\pi\opt \in \PiSD$ from for a sufficiently large horizon $t$ from \cref{thm:erm-main-convergence} and \cref{sec:optim-mark-polic}, we get that
\begin{equation*}
\begin{aligned} 
\bm{\mu}\tr\bm{v}^{\infty}(\pi\opt)
\le
\sup_{\pi\in \PiHR} \liminf_{t \to \infty} \bm{\mu}\tr \bm{v}^t(\pi) 
\le 
\liminf_{t \to \infty} \sup_{\pi\in \PiHR^t} \bm{\mu}\tr \bm{v}^t(\pi) 
\le
\bm{\mu}\tr \bm{v}^{\infty}(\pi\opt), 
\end{aligned}
\end{equation*}
which implies that all inequalities above hold with equality.
\end{proof}

\subsection{Proof of \cref{lem:bounded-contraction-fixedpoint}}
\label{proof:bounded-contraction-fixedpoint}

We use $\bm{p}^{\bm{d}}\in \Real_+^S$ to represent the probability of terminating from any state for each $\bm{d}\in \mathcal{D}$:
\[
 p^{\bm{d}}_s = \sum_{a\in \mathcal{A}} d_a(s) \cdot \bar{p}(s, a, e), \quad \forall s\in \states.
\]
The following lemma establishes a convenient representation of the termination probabilities. 
\begin{lemma} \label{lem:termination-probability}
Assume a policy $\pi = (\bm{d})_{\infty}\in \PiSR$. Then, the probability of terminating in $t \in \mathbb{N}, t > 0$ or fewer steps is
\[
  \sum_{k=0}^{t-1} \bm{\mu}\tr(\bm{P}^{\bm{d}})^k \bm{p}^{\bm{d}}
  \;=\;
  \bm{\mu}\tr (\bm{I} - (\bm{P}^{\bm{d}})^t) \bm{1},
\]
where $\bm{P}^{\bm{d}}$ is a $|\states| \times|\states|$ matrix, and each element $P^{\bm{d}}_{s,s'} =\sum_{a \in \mathcal{A}} d_a(s) \cdot  p(s,a,s') , \forall s,s' \in \states$.
\end{lemma}
\begin{proof}
We have by algebraic manipulation that
\[
\bm{p}^{\bm{d}} = (\bm{I} - \bm{P}^{\bm{d}}) \bm{1}. 
\]
The probability of terminating in step $t > 0$ exactly is
\[
\bm{\mu}\tr(\bm{P}^{\bm{d}})^{t-1} \bm{p}^{\bm{d}}.
\]
Using algebraic manipulation and recognizing a telescopic sum, we have that the probability of terminating in $k \le t$ steps is
\begin{equation*}
\sum_{k=0}^{t-1} \bm{\mu}\tr(\bm{P}^{\bm{d}})^k \bm{p}^{\bm{d}} 
=  \sum_{k=0}^{t-1} \bm{\mu}\tr(\bm{P}^{\bm{d}})^k (\bm{I} - \bm{P}^{\bm{d}}) \bm{1} 
=  \bm{\mu}\tr(\bm{I} - (\bm{P}^{\bm{d}})^t) \bm{1}.
\end{equation*}
\end{proof}

\begin{lemma} \label{lem:exp-prob-monotone}
  For any $\bm{d}\in \mathcal{D}$, the exponential transition matrix is monotone:
  \[
    \bm{x} \ge \bm{y}
    \quad \Rightarrow \quad
    \bm{B}^{\bm{d}} \bm{x} \ge \bm{B}^{\bm{d}} \bm{y},
    \qquad \forall \bm{x}, \bm{y} \in \Real^S.
  \]
\end{lemma}
\begin{proof}
The result follows immediately from the fact that $\bm{B}^{\bm{d}}$ is a non-nonnegative matrix.
\end{proof}

\begin{lemma} \label{lem:pro-goal-state-zero}
For each $t \in \mathbb{N}$ and each policy $\pi = (\bm{d})_{\infty} \in \PiSR$ and each $\bm{\eta}\in \probs{S}$:
\begin{equation}\label{eq:exponential-zero-equivalence}
  \bm{\eta}\tr (\bm{B}^{\bm{d}})^t \bm{b}^{\bm{d} }= 0
  \qquad \Leftrightarrow \qquad
  \bm{\eta}\tr (\bm{P}^{\bm{d}})^t \bm{p}^{\bm{d}} = 0.
\end{equation}
\end{lemma}
\begin{proof}
Algebraic manipulation from the definition in~\eqref{eq:exponential-definitions} shows that
 \begin{equation} \label{eq:goal-squeeze}
   \begin{array}{rcl}
   c_{\mathrm{l}} \cdot  \bm{p}^{\bm{d}}
   \; \le \;&  
   \bm{b}^{\bm{d}}
   &\; \le \;  
   c_{\mathrm{u}} \cdot  \bm{p}^{\bm{d}}, \\
   c_{\mathrm{l}} \cdot  \bm{P}^{\bm{d}} \bm{x}
   \; \le\;& 
   \bm{B}^{\bm{d}} \bm{x}
   &\;\le\; 
   c_{\mathrm{u}} \cdot  \bm{P}^{\bm{d}} \bm{x},
   \qquad
   \forall\bm{x} \in \Real^S , 
   \end{array}
\end{equation}
where
\begin{equation*}
    \begin{aligned}
     & c_{\mathrm{l}} := \min_{s,s'\in \bar{\mathcal{S}}, a\in \mathcal{A}} \exp{ - \beta \cdot  r(s, a, s')} , \\
    & c_{\mathrm{u}} := \max_{s,s'\in \bar{\mathcal{S}}, a\in \mathcal{A}} \exp{ - \beta \cdot  r(s, a, s') }.  
    \end{aligned}
\end{equation*}

Note that $\infty > c_{\mathrm{u}} > c_{\mathrm{l}} > 0$.

We now extend the inequalities in~\eqref{eq:goal-squeeze} to multiple time steps. Suppose that $\bm{y}_{\mathrm{l}} \le \bm{y} \le \bm{y}_{\mathrm{u}}$, then, for $t = \mathbb{N}$:
\begin{equation} \label{eq:goal-squeeze-t}
  c_{\mathrm{l}}^t \cdot (\bm{P}^{\bm{d}})^t \bm{y}_{\mathrm{l}}
  \; \le\;  (\bm{B}^{\bm{d}})^t \bm{y}
  \; \le\;
  c_{\mathrm{u}}^t \cdot (\bm{P}^{\bm{d}})^t \bm{y}_{\mathrm{u}}.
\end{equation}
For the left inequality in~\eqref{eq:goal-squeeze-t}, the induction proceeds as follows. The base case $t = 0$ holds immediately. For the inductive step, suppose that the left inequality in~\eqref{eq:goal-squeeze-t} property holds for $t \in \mathbb{N}$  then it also holds for $t+1$ for each $\bm{y}\in \Real^S$ as
\begin{equation*}
\begin{aligned}
         (\bm{B}^{\bm{d}})^{t+1} \bm{y} 
  \; =\; 
    \bm{B}^{\bm{d}}(\bm{B}^{\bm{d}})^t \bm{y} 
  \; \ge\;   c_{\mathrm{l}}^t\bm{B}^{\bm{d}}(\bm{P}^{\bm{d}})^t \bm{y}_{\mathrm{l}} 
  \; \ge\; 
 c_{\mathrm{l}}^{t+1}\bm{P}^{\bm{d}}(\bm{P}^{\bm{d}})^t \bm{y}_{\mathrm{l}} 
  \; =\; 
 c_{\mathrm{l}}^{t+1}(\bm{P}^{\bm{d}})^{t+1} \bm{y}_{\mathrm{l}}.
 \end{aligned}
\end{equation*}

Above, the first inequality follows from \cref{lem:exp-prob-monotone} and from the inductive assumption, and the second inequality follows from~\eqref{eq:goal-squeeze-t} by setting $\bm{x} = \bm{P}^{\bm{d}}\bm{y}$. The right inequality in~\eqref{eq:goal-squeeze-t} follows analogously.

Exploiting the fact that $\bm{\eta} \ge \bm{0}$ and substituting $\bm{y} = \bm{b}^{\bm{d}}$, $\bm{y}_{\mathrm{l}} = c_{\mathrm{l}} \cdot \bm{p}^{\bm{d}}$, $\bm{y}_{\mathrm{u}} = c_{\mathrm{u}}\cdot \bm{p}^{\bm{d}}$ into~\eqref{eq:goal-squeeze-t} and using the bounds in~\eqref{eq:goal-squeeze}, we get that
\begin{equation*} 
\begin{aligned}
  0
  \;  \le\; 
  c_{\mathrm{l}}^{t+1} \cdot \bm{\eta}\tr (\bm{P}^{\bm{d}})^t \bm{p}^{\bm{d}} 
  \; \le\;  \bm{\eta}\tr (\bm{B}^{\bm{d}})^t \bm{b}^{\bm{d}} 
  \; \le\;
  c_{\mathrm{u}}^{t+1} \cdot \bm{\eta}\tr (\bm{P}^{\bm{d}})^t \bm{p}^{\bm{d}}, 
\end{aligned}
\end{equation*}
where the terms are non-negative because all constants, matrices, and vectors are non-negative. Therefore,
\[
 \bm{\eta}\tr (\bm{B}^{\bm{d}})^t \bm{b}^{\bm{d}} = 0
 \quad\Leftrightarrow \quad
  \bm{\eta}\tr (\bm{P}^{\bm{d}})^t \bm{p}^{\bm{d}} = 0.
\]
\end{proof}

\begin{lemma} \label{lem:exponential-eigen-non-zero}
  For each $\pi = (\bm{d})_{\infty} \in \PiSR$, there exists an $\bm{f}\in \Real^S$ such that
  \[
    \bm{f}\tr \bm{B}^{\bm{d}} = \rho(\bm{B}^{\bm{d}}) \cdot \bm{f}\tr,
    \quad \bm{f} \ge \bm{0},
    \quad \bm{f} \neq \bm{0},
   \quad \bm{f}\tr \bm{b}^{\bm{d}}  > 0.
  \]
\end{lemma}
\begin{proof}
Because $\bm{B}^{\bm{d}}$ is non-negative, there exists an $\bm{f} \in \Real^S$ from the Perron-Frobenius theorem~\cite[Theorem~8.3.1]{Horn2013}, such that
  \[
    \bm{f}\tr \bm{B}^{\bm{d}} = \rho(\bm{B}^{\bm{d}}) \cdot \bm{f}\tr,
    \qquad \bm{f} \ge \bm{0},
    \qquad \bm{f} \neq \bm{0}.
  \]
  Furthermore, because  $\bm{b}^{\bm{d}} \ge  \bm{0}$, we also have that
  \[
    \bm{f}\tr \bm{b}^{\bm{d}} \ge \bm{0}.
  \]
To prove the theorem, it remains to show that $\bm{f}\tr \bm{b}^{\bm{d}} \neq 0$, which we do by deriving a contradiction. Note that the $\bm{f}$ constructed from the Perron-Frobenius theorem is scale-independent. Therefore, without loss of generality, assume that $\bm{1}\tr \bm{f} = 1$ and suppose that $\bm{f}\tr \bm{b}^{\bm{d}} = 0$. Then:
\begin{align*}
\bm{f}\tr \bm{b}^{\bm{d}} &= 0 \\
   \bm{f}\tr (\bm{B}^{\bm{d}})^k \bm{b}^{\bm{d}} &= 0, \, \forall k \in \mathbb{N} &&  \text{[ from } \bm{f}\tr \bm{B}^{\bm{d}} = \rho(\bm{B}^{\bm{d}}) \cdot \bm{f}\tr \text{]} \\
   \bm{f}\tr (\bm{P}^{\bm{d}})^k \bm{p}^{\bm{d}}&= 0, \, \forall k \in \mathbb{N} &&  \text{[ from \cref{lem:pro-goal-state-zero} ]} \\
   \sum_{k=0}^{t-1}  \bm{f}\tr (\bm{P}^{\bm{d}})^k \bm{p}^{\bm{d}} &= 0, \, \forall t \in \mathbb{N}, t > 0 && \text{[ by summing elements ]}   \\
    \bm{f}\tr (\bm{I} - (\bm{P}^{\bm{d}})^t)\bm{1} &= 0, \, \forall t \in \mathbb{N} && \text{[ from \cref{lem:termination-probability} ]} \\
   \lim_{t\to \infty} \bm{f}\tr (\bm{I} - (\bm{P}^{\bm{d}})^t)\bm{1} &= 0,  &&  \text{[ limit ]}\\
    \bm{f}\tr \bm{1} &= 0,  && \text{[ from \cref{lem:transient-spectral-radius} ]} 
\end{align*}
which is a contradiction with $\bm{1}\tr \bm{f} \neq  0$. The last step in the derivation follows from $\rho(\bm{P}^{\bm{d}}) < 1$ and therefore $\lim_{t\to \infty} (\bm{P}^{\bm{d}})^t = \bm{0}$~\citep[Theorem~4.5]{Kallenberg2021markov}.
\end{proof}

\begin{proof}[Proof of \cref{lem:bounded-contraction-fixedpoint}]
From \cref{lem:exponential-eigen-non-zero}, there exists an $\bm{f} \in \Real_+^{S}$ that $\bm{f}\tr \bm{B}^{\bm{d}} = \rho(\bm{B}^{\bm{d}}) \cdot \bm{f}\tr$ and $\bm{f} \ge \bm{0}, \bm{f} \neq \bm{0}$. Then from~\eqref{eq:exp-value-expression}:
\begin{equation*}
\begin{aligned}
    -\infty &<  \bm{\eta}\tr  \bm{w}^t(\pi, \bm{z}) \\
    &=  - \bm{\eta}\tr  (\bm{B}^{\bm{d}})^t \bm{z} - \bm{\eta}\tr \sum_{k=0}^{t-1}  (\bm{B}^{\bm{d}})^k \bm{b}^{\bm{d}} \\
    &\le  - \sum_{k=0}^{t-1}  \rho(\bm{B}^{\bm{d}})^k \bm{\eta}\tr \bm{b}^{\bm{d}}.
\end{aligned}
\end{equation*}
The second inequality follows because $\bm{z} \ge \bm{0}$ and $\bm{B}^{\bm{d}}$ is non-negative. Since $\bm{\eta}\tr \bm{b} > 0$ from \cref{lem:exponential-eigen-non-zero}, we can cancel it from the inequality getting that
\[
  \sum_{k=0}^{t-1}  \rho(\bm{B}^{\bm{d}})^k  < \infty.
\]
Then $\rho(\bm{B}^{\bm{d}}) < 1$ because $\rho(\bm{B}^{\bm{d}})\ge 0$ and the geometric series above is bounded.
\end{proof}

\section{Proofs of \cref{sec:ssp-with-evar}}

\subsection{Helper Lemmas}

\begin{lemma} \label{lem:erm-convergence-0}
  The functions $g$ defined in~\eqref{eq:g-definitions} satisfy for each $\pi\in \PiSD$ and $t \in \mathbb{N}$ that
  \begin{align*}
  \lim_{\beta\to 0} g_t(\pi, \beta) &= g_t(\pi, 0), &
  \lim_{\beta\to 0} g_t\opt (\beta) &= g_t\opt(0), \\
  \lim_{\beta\to 0} g_{\infty}(\pi, \beta) &= g_{\infty}(\pi, 0), &
  \lim_{\beta\to 0} g_{\infty}\opt (\beta) &= g_{\infty}\opt(0).
  \end{align*}
\end{lemma}
\begin{proof}
The result for $g_t$ holds from the continuity of ERM as $\beta \to 0^+$; e.g. ~\citet[remark~2.8]{ahmadi2017analytical}. The result for $g_t\opt$ then follows from the fact that $\PiSD$ is finite. Then, \citet[lemma~1]{patek2001terminating} shows that there exists $\beta_0 > 0$ such that $g_t \to g_{\infty}$ uniformly on $(0, \beta_0)$ and the limits can be exchanged. Therefore, the result also holds for $g_{\infty}$ and $g_{\infty}\opt$.
\end{proof}

\begin{lemma} \label{lem:erm-decreasing}
  The functions $g$ defined in~\eqref{eq:g-definitions} are monotonically non-increasing in the parameter $\beta \ge 0$.
\end{lemma}
\begin{proof}
This well-known result follows, for example, from the dual representation of ERM~\cite{hau2023entropic}.
\end{proof}

The following auxiliary lemmas will be used to bound the approximation error of $h$.
\begin{lemma} \label{lem:case1}
Suppose that $f\colon \Real_{+} \to \bar{\Real}$ is non-increasing and $0 < \beta_0$. Then: 
\begin{equation*}
\begin{aligned}
 \sup_{\beta \in [0,\beta_0)} (f(\beta) + \beta^{-1} \log \alpha)  -
(f(\beta_0) + \beta_0^{-1} \log \alpha)  
& \;\le\;  f(0) - f(\beta_0) .
\end{aligned}
\end{equation*}

\end{lemma}

\begin{proof}
  Follows immediately from the fact that $f(\beta) \le f(0), \forall \beta \ge 0$ and $\beta^{-1} \log \alpha < 0, \forall \beta > 0$. 
\end{proof}

\begin{lemma} \label{lem:case2}
Suppose that $f\colon \Real_{+} \to \bar{\Real}$ is non-increasing and $0 < \beta_k < \beta_{k+1}$. Then: 
\begin{equation*}
\begin{aligned}
\sup_{\beta \in [\beta_k,\beta_{k+1})} (f(\beta) + \beta^{-1} \log \alpha) -
(f(\beta_k) + \beta_k^{-1} \log \alpha) 
& \;\le\;  (\beta_{k+1}^{-1} - \beta_k^{-1}) \cdot \log \alpha. 
\end{aligned}
\end{equation*}
\end{lemma}

\begin{proof}
  The proof uses the same technique as the lemma D.5 from~\cite{hau2023entropic}. We restate the proof here for completeness.

  Using the fact that $f$ is non-increasing and $\beta \mapsto \beta^{-1} \log \alpha$ is increasing we get that:
\begin{gather*}
  \sup_{\beta\in [\beta_k,\beta_{k+1})}  \left(  f(\beta) + \beta^{-1}\cdot \log (\alpha) \right) - 
      \left(  f(\beta_k) + \beta_k^{-1}\cdot \log \alpha \right)  \\
    \le \sup_{\beta\in [\beta_k,\beta_{k+1})}  \left(  f(\beta) + \beta_{k+1}^{-1}\cdot \log \alpha \right) - \left( f(\beta_k) + \beta_k^{-1}\cdot \log \alpha \right) \\
    \le \sup_{\beta\in [\beta_k,\beta_{k+1})}  \left(  f(\beta)  - f(\beta_k) \right) +  \left(\beta_{k+1}^{-1}\cdot \log (\alpha)  - \beta_k^{-1}\cdot \log \alpha \right) \\
\le  \beta_{k+1}^{-1}\cdot \log \alpha  - \beta_k^{-1}\cdot \log \alpha .
\end{gather*}
\end{proof}

\begin{lemma} \label{lem:case3}
Suppose that $f\colon \Real_{+} \to \bar{\Real}$ is non-increasing and $0 < \beta_K$. Then:
\begin{equation*}
    \begin{aligned}
       & \sup_{\beta\in [\beta_K, \infty)} (f(\beta) + \beta^{-1} \log  \alpha ) -
  (f(\beta_K) + \beta_K^{-1} \log \alpha) \; \le\;
\frac{-\log (\alpha)}{\beta_K}. 
    \end{aligned}
\end{equation*}
\end{lemma}

\begin{proof}
The proof uses the same technique as the lemma D.6 from~\cite{hau2023entropic}. We restate the proof here for completeness.
  
Because $f$ non-increasing and  $\beta^{-1}\cdot \log \alpha \le 0, \forall \beta > 0$ we get that:
\begin{align*}
 \sup_{\beta\in [\beta_K, \infty)} (f(\beta) + \beta^{-1} \log  \alpha ) -
  (f(\beta_K) + \beta_K^{-1} \log \alpha) 
\le \sup_{\beta\in [\beta_K, \infty)} (f(\beta)  ) -
  (f(\beta_K) + \beta_K^{-1} \log \alpha) 
\le
\frac{-\log (\alpha)}{\beta_K}.
\end{align*}
\end{proof}

\begin{lemma} \label{lem:beta-b-approximation}
Suppose that $f\colon \Real_{+} \to \bar{\Real}$ is non-increasing and $\mathcal{B}(\beta, \delta)$ defined in~\eqref{eq:b-set-defs}. Then for each $\beta_0 > 0$ and $\delta > 0$:
\begin{equation*}
\begin{aligned}
\max_{\beta\in \mathcal{B}(\beta_0, \delta)}  f(\beta) + \beta^{-1} \log \alpha 
& \; \le\; 
\sup_{\beta > 0 } f(\beta) + \beta^{-1} \log  \alpha \\
&  \; \le\; 
\max_{\beta\in \mathcal{B}(\beta_0, \delta)}  f(\beta) + \beta^{-1} \log \alpha +
\max \left\{ f(0) - f(\beta_0) , \delta \right\} .
\end{aligned}
\end{equation*}

\end{lemma}
\begin{proof}
The result follows by algebraic manipulation from \cref{lem:case1,lem:case2,lem:case3}.
\end{proof}

\subsection{Proof of \cref{prop:choice-K}}

\begin{proof}[Proof of \cref{prop:choice-K}]
  First, the iterates $\beta_k, k\ge 0$ can be bounded as
  \begin{align*}
    \beta_{k+1}
    &= \frac{\beta_k \log \frac{1}{\alpha}}{\log\frac{1}{ \alpha } -\beta_k \delta}
              \ge \beta_k \frac{\log \frac{1}{\alpha}}{\log \frac{1}{\alpha} - \beta_0 \delta} \\
    &=  \beta_0 \left(\frac{\log \frac{1}{\alpha}}{\log \frac{1}{\alpha} - \beta_0 \delta} \right)^{k+1}.
\end{align*}
Using the lower bound on $\beta_k$, to show that
\begin{equation*}
\beta_K \ge  \frac{\log \frac{1}{\alpha}}{\delta},
\end{equation*}
it is sufficient to show that
\[
\beta_0 \left(\frac{\log \frac{1}{\alpha}}{\log \frac{1}{\alpha} - \beta_0 \delta} \right)^K \ge 
\frac{\log \frac{1}{\alpha}}{\delta}.
\]
Using the variable $z$, the sufficient condition translates to 
\begin{align*}
  \left( \frac{1}{1-z} \right)^K &\ge \frac{1}{z} \\
  K &\ge \frac{\log  z}{\log (1-z)},
\end{align*}
since $z \in (0,1)$ by the assumption $\beta_0 \delta < \log \frac{1}{\alpha}$.
\end{proof}
\subsection{Proof of \cref{thm:optimal-evar-erm}}
\label{proof:thm-optimal-evar-erm}
We are now ready to prove the main theorem.
\begin{proof}[Proof of \cref{thm:optimal-evar-erm}]
From \cref{lem:erm-convergence-0}, we can choose a $\beta_0 > 0$ be such that $g_{\infty}\opt(0) - g\opt_{\infty} (\beta_0) \le \delta$.

Then we can upper bound the objective as:
\begin{align*}
\sup_{\pi\in \PiHR} \lim_{t\to \infty }  \sup_{\beta > 0} h_t(\pi ,\beta )
&\stackrel{\text{(a)}}{\le} \lim_{t\to \infty }  \sup_{\beta > 0}\sup_{\pi\in \PiHR}  h_t(\pi ,\beta )  \\
& \stackrel{\text{(b)}}{=} 
  \lim_{t\to \infty }  \sup_{\beta > 0}  h_t\opt(\beta ) \\
&\stackrel{\text{(c)}}{\le} 
  \lim_{t\to \infty }  \max_{\beta \in \mathcal{B}(\beta_0, \delta)}  h_t\opt(\beta ) + \delta \\
& \stackrel{\text{(d)}}{\le} 
    \max_{\beta \in \mathcal{B}(\beta_0, \delta)}  \lim_{t\to \infty } h_t\opt(\beta ) + \delta\\
&\le \max_{\beta \in \mathcal{B}(\beta_0, \delta)}  h_{\infty }\opt(\beta ) + \delta.
\end{align*}
Here, (a) follows because the limit of an upper bound on a sequence is an upper bound on the limit, (b) follows from the definition of $h_t\opt $, (c) from \cref{lem:beta-b-approximation}, (d) from the continuity of $\max$ over finite sets.

Let $\pi\opt \in \PiSD, \beta\opt \in \mathcal{B}(\beta_0, \delta)$ be the maximizers that attain the objective above:
\[
  h_{\infty }(\pi\opt, \beta\opt ) = \max_{\beta \in \mathcal{B}(\beta_0, \delta)}  h_{\infty }\opt(\beta ).
\]
Then, continuing the upper bound on the objective:
\begin{align*}
 \max_{\beta \in \mathcal{B}(\beta_0, \delta)}  h_{\infty }\opt(\beta ) + \delta 
  &= h_{\infty }(\pi\opt, \beta\opt )  + \delta  \\
  &= \lim_{t\to \infty } h_t(\pi\opt, \beta\opt )  + \delta  \\
  &\le  \lim_{t\to \infty } \sup_{\beta > 0} h_t(\pi\opt, \beta)  + \delta  \\
 & \le \lim_{t\to \infty }  \evaro_{\alpha}^{\pi\opt, \bm{\mu}} \left[ \sum_{k=0}^{t-1} r(\tilde{s}_k,\tilde{a}_k,\tilde{s}_{k+1}) \right] + \delta ,
\end{align*}
which shows that $\pi\opt$ is $\delta$-suboptimal since $\displaystyle \lim_{t\to \infty }  \evaro_{\alpha}^{\pi\opt, \bm{\mu}} \left[ \sum_{k=0}^{t-1} r(\tilde{s}_k,\tilde{a}_k,\tilde{s}_{k+1}) \right]$ is a lower bound on the objective. 
\end{proof}

\subsection{Proof of \cref{cor:evar-stationary-policy}}
\label{proof:cor-evar-stationary-policy}

\begin{proof}[Proof of \cref{cor:evar-stationary-policy}]
Suppose that there exists a $\hat{\pi}\in \PiHR \setminus \PiSD$ such that it attains an objective value greater than the best $\pi\in \PiSD$ by at least $\epsilon > 0$. Note that a best $\pi\in \PiSD$ exists because $\PiSD$ is finite. Then, one can derive a contradiction by choosing a $\pi\in \PiSD$ that is at most $\epsilon / 2$ suboptimal.   
\end{proof}

\section{Additional Material of \cref{sec:numerical-eval}}
\label{sec:addit-mater}

The detailed description of the gambler's ruin is as follows. The state space is $\bar{\states} = \{0,1,\cdots,7,e \}$ and the action space is $\actions = \{0,1,\cdots,6 \}$. Action $0$ represents quitting the game, and the other actions represent the bet size. The initial state is chosen randomly according to a uniform distribution over states $1, \dots, 7$.

The transitions and rewards in the MDP are defined as follows.

In state $0$, the gambler has capital $0$, takes the only available action $0$, receives the reward $-1$, and transits to the sink state $e$. That is:
\[
  \bar{p}(0,0,e)=1, \qquad
  \bar{r}(0,0,e)=-1.
\]

If state $7$, the gambler has capital 7, reaches the target wealth level, takes the only available action $0$, receives the reward $7$, and transits to the sink state $e$. That is:
\[
  \bar{p}(7,0,e)=1, \qquad
    \bar{r}(7,0,e)=7.
\]

For the sink state $e$, the only available action is $0$, and the transition probabilities and rewards are
\[
  \bar{p}(e,0,e)=1, \qquad
  \bar{r}(e,0,e)=0.
\]

For other states $s\in \left\{ 1, \dots , 6 \right\}$, the gambler's capital is $s$, can take any action $a \in \{ 1,\cdots,s \}$ to continue the game or an action $0$ to quit the game.

If the game is won, the gambler's bet doubles. If the game is lost, the gambler's bet is lost. That is, the transition probabilities and rewards for $a\in \left\{ 1, \dots , s \right\}$ are
\[
  \bar{p}(s,a,s') =
  \begin{cases}
      q &\text{if } s' = \min \left\{  s + a, 7 \right\}, \\
    1-q &\text{if } s' = s - a, \\
      0 &\text{otherwise},
  \end{cases}
  \qquad
  \bar{r}(s,a,s') = 0,
\]
for each $s' \in  \left\{ 0, \dots, 7, e \right\}$ where $q \in [0,1]$ is the win probability.

If $s\in \left\{ 1, \dots , 6 \right\}$ and the action is $0$, the gambler quits the game and collects the reward. That is:
\[
  \bar{p}(s,0,e) = 1, \qquad
  \bar{r}(s,0,e)=s.
\]

\end{document}